\documentclass{article}

\PassOptionsToPackage{numbers,compress}{natbib}
\usepackage[preprint]{neurips_2026}

\usepackage[utf8]{inputenc}
\usepackage[T1]{fontenc}
\usepackage{hyperref}
\usepackage{url}
\usepackage{booktabs}

\usepackage{amsmath}
\usepackage{amsfonts}
\usepackage{nicefrac}
\usepackage{microtype}
\usepackage{xcolor}
\usepackage{graphicx}
\usepackage[most]{tcolorbox}

\usepackage{amsthm}
\newtheorem{assumption}{Assumption}
\newtheorem{proposition}{Proposition}
\newtheorem{lemma}{Lemma}
\newtheorem{corollary}{Corollary}

\title{Large Language Models as Amortized Pareto-Front Generators for Constrained Bi-Objective Convex Optimization}
\newcommand{\AuthorBlock}[2]{%
  \begin{minipage}[t][4.3em][t]{0.47\textwidth}
    \centering
    \textbf{#1}\\[-0.1em]
    {\normalfont #2}
  \end{minipage}%
}

\newcommand{\AuthorRow}[4]{%
  \makebox[\textwidth][c]{%
    \AuthorBlock{#1}{#2}\hfill
    \AuthorBlock{#3}{#4}%
  }%
}

\newcommand{\AuthorSingle}[2]{%
  \makebox[\textwidth][c]{%
    \begin{minipage}[t][3.8em][t]{0.47\textwidth}
      \centering
      \textbf{#1}\\[-0.1em]
      {\normalfont #2}
    \end{minipage}%
  }%
}

\author{%
  \AuthorRow
    {Peipei Xu\protect\footnotemark[2]}
    {University of Shanghai for Science\\and Technology}
    {SiYuan Ma}
    {Nanyang Technological University}\\[1.45em]
  \AuthorRow
    {Yaohua Liu}
    {Guangdong Institute of Intelligence Science\\and Technology}
    {Yu Wu}
    {Georgia Institute of Technology}\\[1.45em]
  \AuthorRow
    {Guanliang Liu}
    {The University of Michigan}
    {Yang Zhang}
    {The Hong Kong University of Science\\and Technology}\\[1.45em]
  \AuthorSingle
    {Yong Liu\protect\footnotemark[1]}
    {University of Shanghai for Science\\and Technology}
}

\begin{document}

\maketitle
\begingroup
\renewcommand{\thefootnote}{\fnsymbol{footnote}}
\footnotetext[1]{Corresponding author. Email: \texttt{liuyong.seu@163.com}.}
\footnotetext[2]{Email: \texttt{2335051325@st.usst.edu.cn}.}

\endgroup

\begin{abstract}
Generating feasible Pareto fronts for constrained bi-objective continuous optimization is central to multi-criteria decision-making in engineering design, resource allocation, and operations management. Existing approaches typically rely on iterative scalarization, evolutionary search, or problem-specific numerical solvers, requiring repeated optimization and substantial algorithmic design for each problem instance. Although large language models (LLMs) have recently shown promise in optimization, most existing uses rely on intermediate procedures such as code generation, heuristic proposal, or external solver invocation, and are often tailored to discrete combinatorial settings. This paper introduces DIPS, an end-to-end framework that fine-tunes LLMs as amortized Pareto-front generators for constrained bi-objective convex optimization. Given a problem description, DIPS directly outputs an ordered set of feasible continuous decision vectors approximating the Pareto front. To bridge language modeling and continuous optimization, DIPS combines a discretization scheme for continuous outputs, Numerically Grounded Token Initialization (NGTI) for newly introduced numerical tokens, and Three-Phase Curriculum Optimization (TPCO), which progressively aligns structural validity, numerical feasibility, and Pareto-front quality. Across five families of constrained bi-objective convex problems, fine-tuning a 7B-parameter LLM with DIPS achieves normalized hypervolume ratios of 95.29\%--98.18\% relative to reference fronts. With vLLM-accelerated inference, DIPS solves a single instance in as little as 0.16 seconds and outperforms general-purpose and reasoning LLM baselines, including GPT-5 and DeepSeek-R1, under the evaluated setting. These results show that LLMs can serve as effective amortized generators for continuous Pareto-front approximation, suggesting a language-driven alternative to repeated solver invocation for the studied class of optimization problems.
\end{abstract}

\section{Introduction}

Large language models (LLMs) are increasingly being explored beyond language processing, including in mathematical reasoning, scientific computing, and optimization. In optimization, most existing work uses LLMs as \emph{assistive} components: the model proposes prompts or heuristics, writes code, translates natural-language specifications into solver-ready programs, or repeatedly queries an external optimizer~\citep{yang2024opro,liu2024evolution,Ye2024ReEvoLL,ahmaditeshnizi2024optimus,zhang2024solving,astorga2025autoformulation,lv2025gso,yang2025optibench}. Recent surveys summarize this rapidly growing literature~\citep{liu2026systematic,da2025large}. A more ambitious alternative is to use an LLM as an \emph{end-to-end amortized solver} that directly maps a problem description to a solution. Such a formulation could provide a unified optimization interface without requiring users to select algorithms or invoke specialized software.

We study this possibility for \emph{constrained bi-objective convex optimization}. Unlike single-objective optimization, the target is not one optimum but a Pareto frontier that captures the trade-off between competing objectives~\citep{deb2001multiobjective}. This setting is challenging for autoregressive models because the output is a \emph{set} of continuous decision vectors: each vector must be numerically accurate and feasible, while the frontier as a whole must preserve global trade-off geometry. Existing end-to-end LLM optimization work has largely focused on single-solution combinatorial problems, where outputs are discrete and feasibility is easier to express token by token~\citep{jiang2025llm_end2end_co}. These formulations do not directly transfer to continuous Pareto-front generation.

Our experiments show that naive supervised fine-tuning fails in this regime: directly generating raw decimal frontier outputs is poorly learnable, newly introduced numerical tokens are difficult to adapt, and token-level cross-entropy does not reflect numerical proximity, while imposing numerical supervision from the beginning destabilizes training. The main challenge is therefore not only optimization difficulty, but also \emph{how the problem is represented and learned as a language modeling task}.

To address this, we propose \textbf{DIPS}, a unified framework that adapts pretrained LLMs into amortized Pareto-front generators. DIPS combines three ingredients. First, a \emph{discretization scheme} encodes each scalar as a compact two-token numerical representation, reducing sequence length and regularizing the output structure. Second, \emph{Numerically Grounded Token Initialization (NGTI)} warm-starts the newly added numerical tokens by reusing the pretrained model's existing numerical embedding structure. Third, \emph{Three-Phase Curriculum Optimization (TPCO)} starts with structural learning and progressively introduces coarse and fine numerical supervision. At inference time, a lightweight multi-pass fusion procedure further improves frontier coverage. Across five families of randomized constrained bi-objective convex problems, DIPS achieves strong hypervolume and IGD$^+$, maintains near-perfect feasibility, generalizes to unseen dimensions and constraints, and outperforms substantially larger untuned general-purpose LLMs as well as strong classical baselines.

Our contributions are threefold: \textbf{(1)} We formulate end-to-end Pareto-front generation for constrained bi-objective convex optimization as an autoregressive language modeling problem, casting an LLM as an amortized solver that maps problem descriptions to ordered sets of feasible decision vectors. \textbf{(2)} We propose DIPS, which combines a discretization scheme, Numerically Grounded Token Initialization, and Three-Phase Curriculum Optimization to make this setting learnable for pretrained LLMs. \textbf{(3)} DIPS substantially outperforms LLM-based baselines and is competitive with strong classical scalarization methods under a matched wall-clock budget.

\section{Related Work}

\textbf{Bi-objective Continuous Optimization.}
Bi-objective continuous optimization seeks a Pareto frontier in a continuous feasible region under box and linear constraints. Classical approaches include scalarization, normal-boundary or normal-constraint constructions, decomposition, Bayesian optimization, and evolutionary algorithms~\citep{das1998normal,deb2002fast,zhang2007moead,messac2003normal,daulton2023hvkg,wang2024hvi_distribution}. Learning-based approaches instead approximate solution maps, Pareto sets, or distributions over high-quality solutions to amortize repeated optimization cost~\citep{lin2022pareto,zhang2023hypervolume_geom,tuan2024framework,xue2024offline_moo,yuan2025paretoflow,shrestha2026pcd,annadani2025pgdmoo,hotegni2026spread}. We instead investigate whether a pretrained LLM can directly generate a feasible Pareto frontier as a text sequence.

\textbf{LLMs for Optimization.}
Recent LLM-based optimization methods fall into two categories. The first uses LLMs indirectly, e.g., for code generation, heuristic design, or translating problem descriptions into solver-ready programs~\citep{liu2024evolution,Ye2024ReEvoLL,ahmaditeshnizi2024optimus,zhang2024solving,astorga2025autoformulation,lv2025gso,yang2025optibench,chen2022towards}. The second studies end-to-end generation, primarily for single-objective combinatorial problems with discrete outputs~\citep{jiang2025llm_end2end_co}. Our setting differs substantially: the target is a structured set of feasible continuous solutions that must recover the geometry of a Pareto frontier.

\textbf{LLMs for Numerical Generation.}
A growing body of work studies the numerical limitations of language models and improves them via tokenization, embedding initialization, or modified training objectives~\citep{singh2024tokenization_counts,golkar2023xval,zhou2025fone,schwartz2024numerologic,jin2026geonum,yang2025number_cookbook,feher2025dynamic_tokenization,schmidt2024tokenization,piskorz2026numerical_distributions,kreitner2025bit_tokens,li2025numericbench,dugan2024occamllm}. We build on this line of work, but focus on a setting where numerical precision, feasibility, and set-level structure must be learned jointly for Pareto-front generation.

\section{Problem Statement}

We study end-to-end Pareto-front generation with LLMs for a family of constrained bi-objective continuous optimization problems. Let $\mathcal{P}$ denote the problem family, and let each instance $p \in \mathcal{P}$ be defined as $p = \big(n, f_{p,1}, f_{p,2}, \ell_p, u_p, A_p, b_p\big)$, where $f_{p,1},f_{p,2}:\mathbb{R}^n \rightarrow \mathbb{R}$ are two objectives to be minimized over the feasible set
\begin{equation}
\mathcal{X}_p = \big\{x \in \mathbb{R}^n \mid \ell_p \le x \le u_p,\; A_p x \le b_p \big\}.
\end{equation}
In our experiments, $(f_{p,1}, f_{p,2})$ is drawn from five problem families: general bi-objective quadratic programs, separable quadratic programs, ridge-type least-squares objectives, Huber regression objectives, and softplus-based smooth convex objectives.

A feasible point $x^\star \in \mathcal{X}_p$ is \emph{Pareto optimal} if no other feasible $x$ satisfies $f_{p,1}(x) \le f_{p,1}(x^\star)$ and $f_{p,2}(x) \le f_{p,2}(x^\star)$ with at least one strict inequality. For each instance $p$, we represent the target frontier as an ordered set of $K$ Pareto solutions $\mathcal{S}_p^\star = \{x_p^{(1)},\ldots,x_p^{(K)}\}$ obtained by an external solver at $K$ sampled trade-off points; we take $K=20$.

Let $\pi_\theta : \mathcal{T} \rightarrow \mathcal{T}$ denote an LLM, $\phi : \mathcal{P} \rightarrow \mathcal{T}$ the textualization of an instance, and $\psi_p : \mathcal{T} \rightarrow \mathbb{R}^{K \times n}$ the parser that recovers $K$ candidate vectors. End-to-end frontier generation is then $\hat{\mathcal{S}}_p = \psi_p\big(\pi_\theta(\phi(p))\big)$, where each predicted $\hat{x} \in \hat{\mathcal{S}}_p$ is feasible if $\ell_p \le \hat{x} \le u_p$ and $A_p \hat{x} \le b_p$.

We evaluate generated frontiers along three axes: \emph{feasibility}, \emph{dominated volume recovery}, and \emph{reference-coverage error}. Throughout, $\mathrm{PF}(\cdot)$ denotes the non-dominated subset of a finite set of objective vectors, and $\hat{\mathcal{S}}_p^{\mathrm{feas}} \subseteq \hat{\mathcal{S}}_p$ the subset of feasible predictions.

\paragraph{Feasibility rate.}
Because solutions come from rounded discretized tokens, feasibility uses a fixed tolerance $\tau$: $\hat{x}\in\hat{\mathcal{S}}_p$ is feasible iff $\ell_p-\tau \le \hat{x} \le u_p+\tau$ and $A_p\hat{x} \le b_p+\tau$ component-wise. Let $\mathcal{X}_p^{\tau}$ be the resulting feasible set. The point-wise feasibility rate on a test set $\mathcal{P}_s \subset \mathcal{P}$ is
\begin{equation}
M_f(\mathcal{P}_s) =
\frac{1}{|\mathcal{P}_s|}
\sum_{p \in \mathcal{P}_s}
\frac{1}{|\hat{\mathcal{S}}_p|}
\sum_{\hat{x} \in \hat{\mathcal{S}}_p}
\mathbf{1}\!\left[\hat{x} \in \mathcal{X}_p^{\tau}\right].
\label{eq:feasibility_rate}
\end{equation}
We use $\tau=5\times 10^{-5}$, consistent with the four-decimal precision of the discretization scheme.

\paragraph{Hypervolume ratio.}
Both the predicted and reference frontiers are mapped to a normalized objective space defined by the reference frontier alone, so that the prediction cannot move the comparison scale. Let $\hat{\mathcal{F}}_p^{\mathrm{feas}}=\{(f_{p,1}(\hat{x}), f_{p,2}(\hat{x})) \mid \hat{x} \in \hat{\mathcal{S}}_p^{\mathrm{feas}}\}$ and $\mathcal{F}_p^\star=\{(f_{p,1}(x),f_{p,2}(x)) \mid x \in \mathcal{S}_p^\star\}$ denote the feasible predicted and reference objective sets. With $z_p^{\mathrm{ideal}}, z_p^{\mathrm{nadir}}$ the component-wise min/max of $\mathrm{PF}(\mathcal{F}_p^\star)$, the reference-based normalization is $\eta_p(y) = (y - z_p^{\mathrm{ideal}}) \oslash (z_p^{\mathrm{nadir}} - z_p^{\mathrm{ideal}})$ ($\oslash$ component-wise division). With a fixed reference point $r=(1.1,1.1)$ in normalized space, the hypervolume ratio is
\begin{equation}
\mathrm{HVR}(p) =
\frac{
\mathrm{HV}\!\left(\eta_p\!\left(\mathrm{PF}(\hat{\mathcal{F}}_p^{\mathrm{feas}})\right);\, r\right)
}{
\mathrm{HV}\!\left(\eta_p\!\left(\mathrm{PF}(\mathcal{F}_p^\star)\right);\, r\right)
},
\label{eq:hvr}
\end{equation}
with $\mathrm{HVR}(p)=0$ whenever no feasible prediction exists.

\paragraph{IGD$^+$.}
IGD$^+$ measures how well the feasible prediction set covers the reference frontier under the same normalization:
\begin{equation}
\mathrm{IGD}^{+}(p) =
\frac{1}{\big|\mathrm{PF}(\mathcal{F}_p^\star)\big|}
\sum_{y \in \mathrm{PF}(\mathcal{F}_p^\star)}
\min_{\hat{y} \in \hat{\mathcal{F}}_p^{\mathrm{feas}}}
\left\|
\max\!\big(\eta_p(\hat{y}) - \eta_p(y),\, 0\big)
\right\|_2.
\label{eq:igdplus}
\end{equation}
The minimum is taken over the full feasible prediction set $\hat{\mathcal{F}}_p^{\mathrm{feas}}$ rather than only its non-dominated subset, because denser predictions yield more faithful coverage estimates; if $\hat{\mathcal{F}}_p^{\mathrm{feas}}=\emptyset$, $\mathrm{IGD}^{+}(p)$ is recorded as $\infty$ and excluded from averaging.

\section{Methodology}
\label{sec:method}

\begin{figure}[t]
\centering
\includegraphics[width=\linewidth]{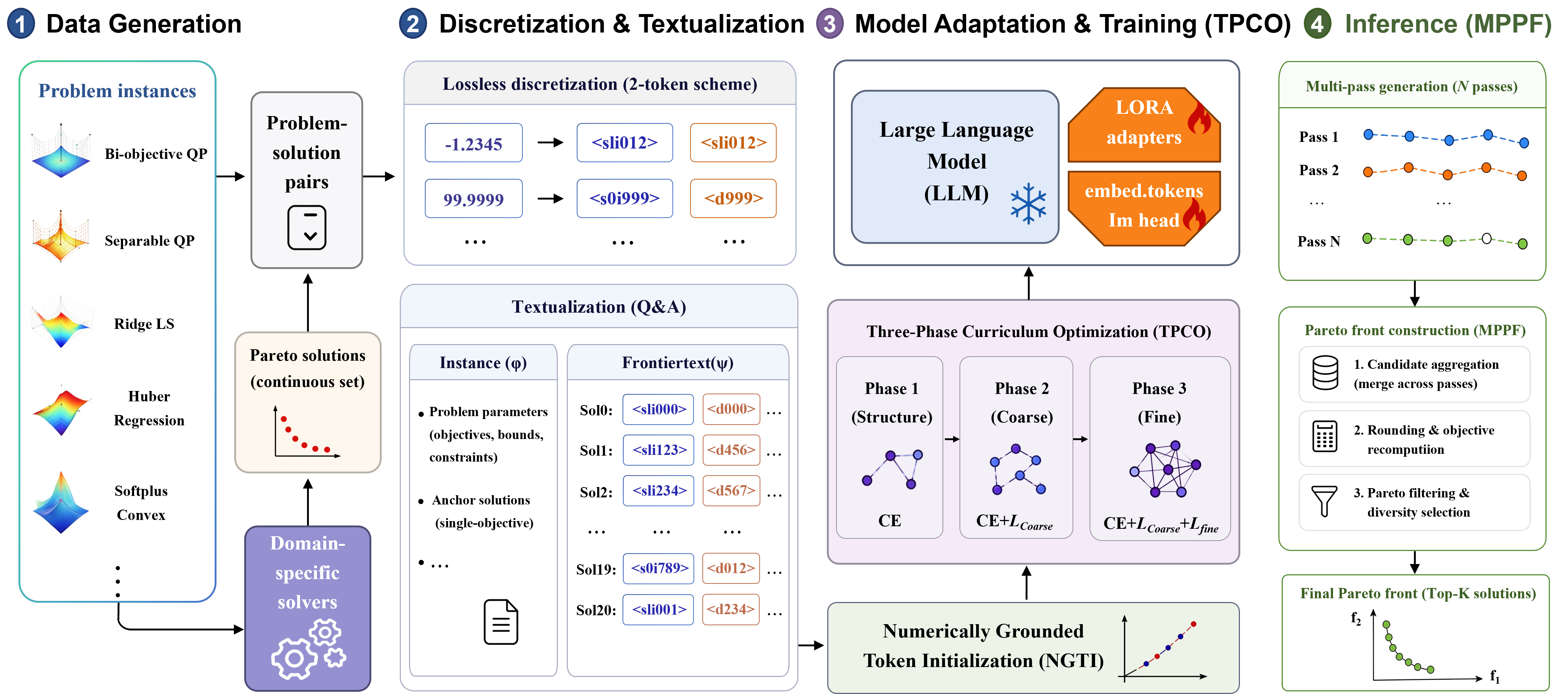}
\caption{Overview of the DIPS framework. Continuous Pareto-front targets are converted into compact two-token sequences via a discretization scheme; newly added numerical tokens are warm-started by Numerically Grounded Token Initialization (NGTI); the LoRA-adapted LLM is then trained under the Three-Phase Curriculum Optimization (TPCO) objective; at inference time, multiple decoding passes are merged into the final frontier by Multi-Pass Pareto Fusion (MPPF).}
\label{fig:framework}
\end{figure}

We propose \textbf{DIPS}, an end-to-end framework that adapts pretrained LLMs into amortized Pareto-front generators. As illustrated in Figure~\ref{fig:framework}, DIPS comprises five parts: (1) constructing problem--frontier supervision pairs, (2) discretizing continuous outputs into compact numerical tokens, (3) initializing the new numerical vocabulary with pretrained numerical structure, (4) training under a coarse-to-fine curriculum that separates structural and numerical learning, and (5) improving test-time frontier coverage via multi-pass fusion.

\subsection{Data Generation}

We construct training pairs by generating instances from five families of constrained bi-objective convex problems: general bi-objective quadratic programs, separable quadratic programs, ridge-type least-squares objectives, Huber regression objectives, and softplus-based smooth convex objectives. They share a common input format consisting of objective parameters, box bounds, and optional linear constraints.

For each family, we sample feasible instances from randomized parameters and solve them offline with family-specific optimizers via an $\epsilon$-constraint procedure. The resulting frontier is filtered and resampled into a fixed ordered set of $K=20$ Pareto solutions, following common practice for training Pareto-set learners against standardized frontier representations~\citep{lin2022pareto,zhang2023hypervolume_geom}. Each instance is then converted into text: the prompt contains the problem parameters together with two single-objective anchor solutions providing compact cues about the two ends of the trade-off surface, and the target output is the ordered Pareto sequence. Full instance-generation, solver, and template details are deferred to Appendix~\ref{app:instance_generation} and Appendix~\ref{app:prompt_template}.

\subsection{Discretization}

Autoregressive LLMs are poorly matched to raw decimal outputs: character-level decimals are long, irregular, and conflate numerical precision with sequence formatting. We therefore replace raw decimal serialization with a two-token encoding tailored to four-decimal continuous values.

For a scalar $v$ rounded to four decimal places, we encode
\begin{equation}
v \mapsto \texttt{<sXiXXX><dXXX>},
\end{equation}
where the first token stores the sign, integer part, and first decimal digit, and the second token stores the remaining three decimal digits. For example,
\[
99.9999 \mapsto \texttt{<s0i999><d999>}, \qquad
-1.2345 \mapsto \texttt{<s1i012><d345>}.
\]
This encoding is exact at the chosen precision and represents every scalar with exactly two tokens, reducing output length by roughly 70\% in our data. We also mark matrix rows and frontier points explicitly (e.g., \texttt{R0:}, \texttt{Sol0:}) to preserve structural locality. The design is motivated by evidence that numerical behavior in LLMs is strongly affected by tokenization and embedding choices~\citep{singh2024tokenization_counts,schmidt2024tokenization,feher2025dynamic_tokenization,golkar2023xval,zhou2025fone,li2025numericbench}.

\subsection{Numerically Grounded Token Initialization (NGTI)}

Discretization introduces a new numerical vocabulary absent from the pretrained tokenizer. Randomly initializing these tokens makes optimization unstable, since the model must learn both their semantics and their role in frontier generation from scratch—consistent with prior findings that vocabulary expansion, initialization, and numerical embeddings materially affect adaptation and numeracy~\citep{mundra2024empirical,golkar2023xval,zhou2025fone,jin2026geonum}. NGTI instead initializes them by reusing the pretrained model's existing numerical embedding structure.

We add two token groups: integer-prefix tokens \texttt{<sXiXXX>} and fractional-suffix tokens \texttt{<dXXX>}. For each, NGTI composes its embedding from pretrained embeddings of basic numerical symbols already in the model—digits, signs, decimal points—using position-dependent weights so that the composed embedding reflects the token's coarse numerical role. Integer-prefix tokens are built from sign, integer digits, decimal point, and the first decimal digit; fractional-suffix tokens from decimal point and the remaining three decimal digits. To preserve local numerical ordering, we further add a small value-dependent perturbation scaled by the standard deviation of the pretrained embedding space, so that numerically nearby tokens start close while distant tokens remain distinguishable. If the output head is untied from the input embedding matrix, it is initialized in the same way.

\subsection{Three-Phase Curriculum Optimization (TPCO)}

Even with discretized outputs, standard cross-entropy (CE) is insufficient: it learns format and token placement but is numerically insensitive, treating a nearby and a distant value as equally wrong whenever token identities differ. Directly imposing numerical losses from the start, however, destabilizes training before the basic frontier structure has been learned. We address this mismatch with \emph{Three-Phase Curriculum Optimization (TPCO)}, which separates structural acquisition from coarse and fine numerical refinement.

Let $y=(y_1,\dots,y_T)$ be the target sequence for input instance $x$, $p_\theta(y_t \mid y_{<t},x)$ the autoregressive distribution, and $\mathcal{M}$ the valid supervision positions. The token-level cross-entropy is $\mathcal{L}_{\mathrm{CE}} = -\sum_{t \in \mathcal{M}} \log p_\theta(y_t \mid y_{<t},x)$. To complement CE with numerical geometry, we add two auxiliary alignment terms over the integer-prefix and fractional-suffix vocabularies $\mathcal{V}_{\mathrm{int}}, \mathcal{V}_{\mathrm{frac}}$, with coarse/fine values $c(j),d(j)$ and target values $c_t^\star,d_t^\star$:
\begin{equation}
\mathcal{L}_{\mathrm{coarse}}
=
\frac{1}{|\mathcal{M}_{\mathrm{int}}|}
\sum_{t \in \mathcal{M}_{\mathrm{int}}}\sum_{j \in \mathcal{V}_{\mathrm{int}}}
p_\theta(j \mid y_{<t},x)\,|c(j)-c_t^\star|,
\label{eq:coarse_loss}
\end{equation}
\begin{equation}
\mathcal{L}_{\mathrm{fine}}
=
\frac{1}{|\mathcal{M}_{\mathrm{frac}}|}
\sum_{t \in \mathcal{M}_{\mathrm{frac}}}\sum_{j \in \mathcal{V}_{\mathrm{frac}}}
p_\theta(j \mid y_{<t},x)\,|d(j)-d_t^\star|.
\label{eq:fine_loss}
\end{equation}
where $\mathcal{M}_{\mathrm{int}},\mathcal{M}_{\mathrm{frac}}$ collect integer-prefix and fractional-suffix positions. The TPCO objective, motivated by curriculum and staged training~\citep{bengio2009curriculum,shen2022staged,chen2025advancing_math_reasoning}, is
\begin{equation}
\mathcal{L}_{\mathrm{TPCO}}
=
\lambda_{\mathrm{CE}}(r)\mathcal{L}_{\mathrm{CE}}
+\lambda_{\mathrm{c}}(r)\mathcal{L}_{\mathrm{coarse}}
+\lambda_{\mathrm{f}}(r)\mathcal{L}_{\mathrm{fine}},
\label{eq:tpco_total}
\end{equation}
with $r\in[0,1]$ the normalized training progress and three phases controlled by milestones $0<r_1<r_2<1$:
\begin{equation}
\big(\lambda_{\mathrm{CE}}(r),\lambda_{\mathrm{c}}(r),\lambda_{\mathrm{f}}(r)\big)
=
\begin{cases}
(1,\;0,\;0), & 0\le r<r_1,\\
\big(1-\tau(1-\lambda_{\min}),\;\tau\beta_{\max},\;0\big), & r_1\le r<r_2,\;\;\tau=\tfrac{r-r_1}{r_2-r_1},\\
\big(\lambda_{\min},\;\beta_{\max},\;\rho\gamma_{\max}\big), & r_2\le r\le 1,\;\;\rho=\tfrac{r-r_2}{1-r_2}.
\end{cases}
\label{eq:tpco_phases}
\end{equation}
In Phase 1 the model learns the serialization format and global frontier structure with CE only; Phase 2 gradually injects coarse alignment to drive probability mass toward the correct numerical region; Phase 3 adds fine alignment to refine the remaining precision. The schedule introduces numerical supervision only after structural competence has emerged, and aligns the curriculum with the coarse/fine decomposition induced by the discretization itself.

\subsection{Multi-Pass Pareto Fusion (MPPF)}
\label{sec:mppf}

A single autoregressive decoding pass may miss parts of the frontier because early token decisions affect later solutions. To improve set-level coverage, we decode multiple times and merge the resulting candidates.

Given an instance $p$, we run the model for $N$ inference passes and pool all parsed candidates $\mathcal{C}_p = \bigcup_{i=1}^{N} \hat{\mathcal{S}}_p^{(i)}$. Each $\hat{x}\in\mathcal{C}_p$ is rounded to four-decimal precision and its objectives recomputed exactly to remove rounding drift. We then deduplicate in two ways: candidates with identical rounded decision vectors are dropped, and candidates whose objective vectors lie within $\tau_{\mathrm{obj}}$ of an already-kept point are also dropped. Infeasible candidates are removed under the same tolerance $\tau$ used in evaluation, leaving $\mathcal{C}_p^{\mathrm{feas}} = \{\hat{x}\in\mathcal{C}_p \mid \hat{x}\in\mathcal{X}_p^{\tau}\}$.

A non-dominated sort on $\mathcal{C}_p^{\mathrm{feas}}$ produces a sequence of fronts $\mathcal{F}_p^{(1)},\mathcal{F}_p^{(2)},\ldots$, and final solutions are selected front-by-front: from $\mathcal{F}_p^{(1)}$ if it has $\ge K$ points, otherwise drawing from subsequent fronts until $K$ feasible points are collected. Within each front, $K$ representatives are chosen by uniform arc-length sampling in normalized objective space:
\begin{equation}
\hat{\mathcal{S}}_p^{\mathrm{final}}
=
\mathrm{ArcSelect}\!\left(
\mathcal{F}_p^{(1)} \cup \mathcal{F}_p^{(2)} \cup \cdots,\, K
\right).
\end{equation}
The fallback to subsequent fronts is rarely triggered when the model is well-trained, but ensures exactly $K$ feasible solutions are returned even when $\mathcal{F}_p^{(1)}$ is sparse.

\section{Experiments}

We use Qwen2.5-7B-Instruct as the backbone model and fine-tune it with LoRA~\citep{hu2022lora} on all five problem families. For each family, we generate $50{,}000$ training instances per dimension from $n=10$ to $n=20$, yielding 11 training dimensions in total. Each target frontier contains $K=20$ Pareto solutions. Unless otherwise stated, the model is trained on discretized problem--frontier pairs with TPCO and decoded with vLLM-accelerated inference.

At test time, we sample with temperature $0.7$. DIPS denotes the base model with single-pass decoding ($N{=}1$); DIPS+MPPF additionally applies multi-pass Pareto fusion ($N{=}4$). We evaluate feasibility (fea.), hypervolume ratio (HV), IGD$^{+}$, and average per-instance inference time, and further study generalization to unseen dimensions and constraints. Full hyperparameters and evaluation details are given in Appendix~\ref{app:exp_settings}.

\subsection{Baselines}
\label{sec:baselines}

We compare DIPS against four groups of baselines, covering both LLM-based and classical multi-objective optimization paradigms. We also report two variants of our method: DIPS, the base fine-tuned model that decodes a single Pareto-front sequence per instance, and DIPS+MPPF, which additionally applies multi-pass Pareto fusion at inference time.

\paragraph{General-purpose LLMs.}
We compare DIPS with strong general-purpose LLMs that all have substantially more parameters than our fine-tuned 7B model, including GPT-5.2 Instant, Claude Sonnet 4.5, Claude Haiku 4.5, Gemini 2.5 Flash, DeepSeek-V3 671B, Qwen3-235B, Llama-3.3-70B, and Qwen2.5-72B-Instruct. Each model is asked to directly generate $K$ feasible decision vectors approximating the Pareto frontier, with a single feasible reference solution included in the prompt to anchor the output format~\citep{jiang2025llm_end2end_co}. API model identifiers and full prompting protocol are provided in Appendix~\ref{app:baseline_details}.

\paragraph{Reasoning LLMs.}
We further compare DIPS with state-of-the-art reasoning models that perform extended chain-of-thought before producing the final output, including GPT-5.2 Thinking, Claude Haiku 4.5 Thinking, DeepSeek-R1~\citep{guo2025deepseek_r1}, and Qwen3-Thinking. They follow the same prompting protocol as the general-purpose LLMs.

\paragraph{Prompt strategies.}
We further include four representative LLM-based optimization prompt strategies: OPRO~\citep{yang2024opro}, which steers solution generation through verbal feedback over previously generated candidates; PHP~\citep{zheng2023php}, which progressively incorporates previous outputs as hints for the next generation; LMEA~\citep{liu2024lmea}, which treats the LLM as an evolutionary operator with crossover- and mutation-style prompts; and SGE~\citep{iklassov2024sge}, which uses the LLM to write code for heuristic or metaheuristic optimizers and executes the generated programs to produce a frontier. For consistency, all prompt-based baselines are implemented on top of DeepSeek-V3 671B.

\paragraph{Classical per-instance optimizers.}
Finally, we include strong classical multi-objective optimization methods that operate directly on exact objectives and constraints: Weighted Sum, Normal Boundary Intersection (NBI), Normal Constraint (NC)~\citep{das1998normal,messac2003normal}, NSGA-II, and MOEA/D~\citep{deb2002fast,zhang2007moead}. These methods are not amortized and are discussed in a dedicated reference study in Section~\ref{sec:classical_reference}.

All methods are evaluated on the same five problem families and the same test instances, using feasibility (fea.), hypervolume ratio (HV), IGD$^+$, and average inference time.

\subsection{Main Results: End-to-End LLM Pareto-Front Generation}
\label{sec:main_results}

We first evaluate whether DIPS can outperform the broad family of LLM-based end-to-end and prompted Pareto-front generators. This section therefore focuses on comparisons with general-purpose LLMs, reasoning LLMs, and prompt strategies; comparisons with classical per-instance optimizers are deferred to Section~\ref{sec:classical_reference}. Reference Pareto frontiers are produced offline by the $\epsilon$-constraint pipeline with family-specific numerical solvers (Appendix~\ref{app:problem_spec}) and are used only for evaluation. Per-family fine-grained results across different decision-variable dimensions are deferred to Appendix~\ref{app:per_family_results}.

\begin{table*}[t]
\centering
\caption{Main results on the five bi-objective continuous optimization families. HV and IGD$^{+}$ are reported as mean$_{\pm\mathrm{std}}$ across test instances; best in bold.}
\label{tab:main_results}
\scriptsize
\setlength{\tabcolsep}{2pt}
\renewcommand{\arraystretch}{0.95}
\resizebox{\textwidth}{!}{%
\begin{tabular}{lccc ccc ccc ccc ccc c}
\toprule
Method
& \multicolumn{3}{c}{SBQP}
& \multicolumn{3}{c}{BOQP}
& \multicolumn{3}{c}{Ridge}
& \multicolumn{3}{c}{Huber}
& \multicolumn{3}{c}{Softplus}
& \begin{tabular}{c}Avg.\\ Time\end{tabular} \\
\cmidrule(lr){2-4} \cmidrule(lr){5-7} \cmidrule(lr){8-10} \cmidrule(lr){11-13} \cmidrule(lr){14-16}
& fea. & HV & IGD$^{+}$
& fea. & HV & IGD$^{+}$
& fea. & HV & IGD$^{+}$
& fea. & HV & IGD$^{+}$
& fea. & HV & IGD$^{+}$ & \\
\midrule

\multicolumn{17}{l}{\textbf{General-purpose Language Models}} \\
GPT-5.2 Instant
& 67\% & 0.512$_{\pm 0.28}$ & 0.412$_{\pm 0.26}$
& 64\% & 0.498$_{\pm 0.29}$ & 0.435$_{\pm 0.27}$
& 71\% & 0.547$_{\pm 0.27}$ & 0.376$_{\pm 0.31}$
& 62\% & 0.483$_{\pm 0.20}$ & 0.451$_{\pm 0.23}$
& 69\% & 0.441$_{\pm 0.32}$ & 0.498$_{\pm 0.27}$ & 6.2s \\
Claude Sonnet 4.5
& 81\% & 0.576$_{\pm 0.20}$ & 0.367$_{\pm 0.22}$
& 78\% & 0.551$_{\pm 0.21}$ & 0.391$_{\pm 0.24}$
& 75\% & 0.603$_{\pm 0.19}$ & 0.328$_{\pm 0.21}$
& 66\% & 0.589$_{\pm 0.22}$ & 0.352$_{\pm 0.25}$
& 64\% & 0.388$_{\pm 0.20}$ & 0.520$_{\pm 0.22}$ & 7.4s \\
Claude Haiku 4.5
& 82\% & 0.494$_{\pm 0.19}$ & 0.428$_{\pm 0.20}$
& 73\% & 0.477$_{\pm 0.20}$ & 0.461$_{\pm 0.21}$
& 51\% & 0.521$_{\pm 0.18}$ & 0.405$_{\pm 0.19}$
& 60\% & 0.462$_{\pm 0.21}$ & 0.486$_{\pm 0.22}$
& 73\% & 0.508$_{\pm 0.18}$ & 0.391$_{\pm 0.20}$ & 9.1s \\
Gemini 2.5 Flash
& 56\% & 0.428$_{\pm 0.22}$ & 0.523$_{\pm 0.25}$
& 53\% & 0.412$_{\pm 0.23}$ & 0.548$_{\pm 0.26}$
& 60\% & 0.456$_{\pm 0.21}$ & 0.487$_{\pm 0.23}$
& 51\% & 0.394$_{\pm 0.24}$ & 0.571$_{\pm 0.27}$
& 58\% & 0.441$_{\pm 0.22}$ & 0.502$_{\pm 0.24}$ & 5.8s \\
DeepSeek-V3 671B
& 59\% & 0.387$_{\pm 0.31}$ & 0.594$_{\pm 0.26}$
& 57\% & 0.371$_{\pm 0.32}$ & 0.618$_{\pm 0.27}$
& 63\% & 0.412$_{\pm 0.30}$ & 0.557$_{\pm 0.24}$
& 55\% & 0.358$_{\pm 0.33}$ & 0.642$_{\pm 0.28}$
& 61\% & 0.396$_{\pm 0.31}$ & 0.578$_{\pm 0.25}$ & 26.3s \\
Qwen3-235B
& 43\% & 0.345$_{\pm 0.20}$ & 0.661$_{\pm 0.27}$
& 41\% & 0.331$_{\pm 0.21}$ & 0.689$_{\pm 0.28}$
& 47\% & 0.371$_{\pm 0.19}$ & 0.624$_{\pm 0.26}$
& 39\% & 0.318$_{\pm 0.22}$ & 0.713$_{\pm 0.22}$
& 45\% & 0.357$_{\pm 0.20}$ & 0.642$_{\pm 0.27}$ & 14.7s \\
Llama-3.3-70B
& 23\% & 0.181$_{\pm 0.18}$ & 0.778$_{\pm 0.21}$
& 21\% & 0.167$_{\pm 0.19}$ & 0.802$_{\pm 0.12}$
& 27\% & 0.205$_{\pm 0.17}$ & 0.741$_{\pm 0.22}$
& 19\% & 0.152$_{\pm 0.20}$ & 0.825$_{\pm 0.11}$
& 11\% & 0.189$_{\pm 0.18}$ & 0.762$_{\pm 0.13}$ & 2.8s \\
Qwen2.5-72B
& 28\% & 0.241$_{\pm 0.16}$ & 0.732$_{\pm 0.21}$
& 26\% & 0.227$_{\pm 0.17}$ & 0.757$_{\pm 0.17}$
& 32\% & 0.265$_{\pm 0.16}$ & 0.798$_{\pm 0.15}$
& 24\% & 0.213$_{\pm 0.18}$ & 0.781$_{\pm 0.17}$
& 20\% & 0.249$_{\pm 0.17}$ & 0.715$_{\pm 0.20}$ & 13.2s \\

\midrule
\multicolumn{17}{l}{\textbf{Reasoning Models}} \\
GPT-5.2 Thinking
& 91\% & 0.578$_{\pm 0.16}$ & 0.351$_{\pm 0.17}$
& 84\% & 0.561$_{\pm 0.17}$ & 0.374$_{\pm 0.18}$
& 87\% & 0.602$_{\pm 0.15}$ & 0.318$_{\pm 0.16}$
& 79\% & 0.545$_{\pm 0.18}$ & 0.391$_{\pm 0.19}$
& 92\% & 0.587$_{\pm 0.16}$ & 0.342$_{\pm 0.17}$ & 1.8m \\
Claude Haiku 4.5 Thinking
& 94\% & 0.553$_{\pm 0.18}$ & 0.382$_{\pm 0.19}$
& 91\% & 0.534$_{\pm 0.19}$ & 0.408$_{\pm 0.20}$
& 88\% & 0.578$_{\pm 0.17}$ & 0.347$_{\pm 0.18}$
& 78\% & 0.518$_{\pm 0.20}$ & 0.428$_{\pm 0.21}$
& 84\% & 0.561$_{\pm 0.18}$ & 0.371$_{\pm 0.19}$ & 2.4m \\
DeepSeek-R1
& 65\% & 0.508$_{\pm 0.19}$ & 0.434$_{\pm 0.21}$
& 62\% & 0.492$_{\pm 0.20}$ & 0.461$_{\pm 0.22}$
& 68\% & 0.531$_{\pm 0.18}$ & 0.401$_{\pm 0.20}$
& 60\% & 0.476$_{\pm 0.21}$ & 0.483$_{\pm 0.23}$
& 66\% & 0.517$_{\pm 0.19}$ & 0.422$_{\pm 0.21}$ & 5.7m \\
Qwen3-Thinking
& 61\% & 0.473$_{\pm 0.20}$ & 0.482$_{\pm 0.22}$
& 58\% & 0.456$_{\pm 0.21}$ & 0.508$_{\pm 0.23}$
& 64\% & 0.497$_{\pm 0.19}$ & 0.448$_{\pm 0.21}$
& 56\% & 0.441$_{\pm 0.22}$ & 0.531$_{\pm 0.24}$
& 62\% & 0.482$_{\pm 0.20}$ & 0.467$_{\pm 0.22}$ & 4.2m \\

\midrule
\multicolumn{17}{l}{\textbf{Prompt Strategies}} \\
OPRO
& 78\% & 0.652$_{\pm 0.14}$ & 0.281$_{\pm 0.15}$
& 75\% & 0.631$_{\pm 0.15}$ & 0.302$_{\pm 0.16}$
& 81\% & 0.681$_{\pm 0.13}$ & 0.254$_{\pm 0.14}$
& 73\% & 0.612$_{\pm 0.16}$ & 0.318$_{\pm 0.17}$
& 79\% & 0.661$_{\pm 0.14}$ & 0.272$_{\pm 0.15}$ & 2.3m \\
PHP
& 75\% & 0.628$_{\pm 0.15}$ & 0.312$_{\pm 0.16}$
& 72\% & 0.607$_{\pm 0.16}$ & 0.334$_{\pm 0.17}$
& 78\% & 0.654$_{\pm 0.14}$ & 0.285$_{\pm 0.15}$
& 70\% & 0.589$_{\pm 0.17}$ & 0.351$_{\pm 0.18}$
& 76\% & 0.638$_{\pm 0.15}$ & 0.301$_{\pm 0.16}$ & 1.7m \\
LMEA
& 73\% & 0.611$_{\pm 0.16}$ & 0.331$_{\pm 0.17}$
& 70\% & 0.591$_{\pm 0.17}$ & 0.354$_{\pm 0.18}$
& 76\% & 0.638$_{\pm 0.15}$ & 0.305$_{\pm 0.16}$
& 68\% & 0.572$_{\pm 0.18}$ & 0.371$_{\pm 0.19}$
& 74\% & 0.621$_{\pm 0.16}$ & 0.321$_{\pm 0.17}$ & 5.4m \\
SGE
& 86\% & 0.741$_{\pm 0.11}$ & 0.198$_{\pm 0.10}$
& 83\% & 0.722$_{\pm 0.12}$ & 0.218$_{\pm 0.11}$
& 88\% & 0.768$_{\pm 0.10}$ & 0.176$_{\pm 0.09}$
& 81\% & 0.703$_{\pm 0.13}$ & 0.234$_{\pm 0.12}$
& 85\% & 0.732$_{\pm 0.11}$ & 0.207$_{\pm 0.10}$ & 3.8m \\

\midrule
\multicolumn{17}{l}{\textbf{Ours}} \\
DIPS
& 97\% & 0.982$_{\pm 0.09}$ & 0.017$_{\pm 0.06}$
& 97\% & 0.969$_{\pm 0.08}$ & 0.023$_{\pm 0.04}$
& 98\% & 0.973$_{\pm 0.09}$ & 0.022$_{\pm 0.04}$
& 99\% & 0.968$_{\pm 0.04}$ & 0.025$_{\pm 0.02}$
& 99\% & 0.953$_{\pm 0.05}$ & 0.033$_{\pm 0.02}$ & 0.16s \\
DIPS+MPPF
& \textbf{100\%} & \textbf{0.998$_{\pm 0.04}$} & \textbf{0.004$_{\pm 0.01}$}
& \textbf{100\%} & \textbf{0.980$_{\pm 0.05}$} & \textbf{0.016$_{\pm 0.02}$}
& \textbf{100\%} & \textbf{0.982$_{\pm 0.04}$} & \textbf{0.014$_{\pm 0.02}$}
& \textbf{100\%} & \textbf{0.972$_{\pm 0.05}$} & \textbf{0.019$_{\pm 0.02}$}
& \textbf{100\%} & \textbf{0.962$_{\pm 0.03}$} & \textbf{0.026$_{\pm 0.02}$} & 0.49s \\
\bottomrule
\end{tabular}%
}
\end{table*}

Several observations follow from Table~\ref{tab:main_results}. First, even the strongest general-purpose LLMs cannot reliably produce feasible high-quality frontiers under direct prompting: feasibility stays below $82\%$ on every family and below $50\%$ for the smaller open-source models, while HV ratios concentrate between $0.17$ and $0.61$. Decoding ordered sets of feasible continuous decision vectors lies outside the operating range of general-purpose chat models, regardless of scale. Second, reasoning models gain only marginal feasibility and frontier quality at substantially higher cost: GPT-5.2 Thinking and Claude Haiku 4.5 Thinking reach $78\%$--$94\%$ feasibility but require $1.8$--$2.4$ minutes per instance, while DeepSeek-R1 and Qwen3-Thinking spend $4$--$6$ minutes per instance and remain below the best non-reasoning baselines (Appendix~\ref{app:reasoning_trace} illustrates why extended chain-of-thought scales poorly to set-valued continuous outputs). Third, prompt-based strategies (OPRO, PHP, LMEA, SGE) close part of the gap, with code-generation prompting (SGE) reaching HV ratios of $0.70$--$0.77$, but each requires $1.7$--$5.4$ minutes per instance.

In contrast, with a fine-tuned 7B model, DIPS attains $97\%$--$99\%$ feasibility and HV ratios of $0.953$--$0.982$ across all five families in $0.16$ seconds per instance, and DIPS+MPPF further pushes feasibility to $100\%$ on every family and HV to $96.2\%$--$99.8\%$ in $0.49$ seconds. Relative to the strongest LLM baseline (SGE), DIPS+MPPF improves HV by more than $0.20$ absolute on every family while cutting inference time by two orders of magnitude. A fine-grained per-dimension analysis across $n\in[10,20]$, which confirms that this advantage holds uniformly over the full dimensional range, is provided in Appendix~\ref{app:per_family_results}.

\subsection{Classical Optimizer Reference Study}
\label{sec:classical_reference}

Classical multi-objective optimizers operate directly on exact objectives and constraints and perform per-instance numerical optimization, so they are not amortized in the sense DIPS is. To make the comparison interpretable, we evaluate every classical baseline under the same per-instance wall-clock budget as DIPS+MPPF, namely $0.5$ seconds, and DIPS itself in single-pass mode ($N{=}1$). Full-budget upper-bound results, in which classical baselines run until their default convergence criteria, are reported in Appendix~\ref{app:fullbudget_classical} as a reference value rather than a fair head-to-head comparison.

\begin{table*}[t]
\centering
\caption{Budget-matched comparison ($0.5$\,s per-instance wall-clock) between classical multi-objective optimizers and DIPS+MPPF, with DIPS in single-pass mode ($N{=}1$); full-budget results in Appendix~\ref{app:fullbudget_classical}.}
\label{tab:budget_matched}
\scriptsize
\setlength{\tabcolsep}{2pt}
\renewcommand{\arraystretch}{0.95}
\resizebox{\textwidth}{!}{%
\begin{tabular}{lccc ccc ccc ccc ccc}
\toprule
Method
& \multicolumn{3}{c}{SBQP}
& \multicolumn{3}{c}{BOQP}
& \multicolumn{3}{c}{Ridge}
& \multicolumn{3}{c}{Huber}
& \multicolumn{3}{c}{Softplus} \\
\cmidrule(lr){2-4} \cmidrule(lr){5-7} \cmidrule(lr){8-10} \cmidrule(lr){11-13} \cmidrule(lr){14-16}
& fea. & HV & IGD$^{+}$
& fea. & HV & IGD$^{+}$
& fea. & HV & IGD$^{+}$
& fea. & HV & IGD$^{+}$
& fea. & HV & IGD$^{+}$ \\
\midrule
Weighted Sum
& 100\% & 0.982$_{\pm 0.01}$ & 0.018$_{\pm 0.01}$
& 100\% & 0.988$_{\pm 0.01}$ & 0.017$_{\pm 0.01}$
& 99\% & 0.984$_{\pm 0.01}$ & 0.018$_{\pm 0.01}$
& 99\% & 0.985$_{\pm 0.01}$ & 0.016$_{\pm 0.01}$
& 99\% & 0.980$_{\pm 0.01}$ & 0.019$_{\pm 0.01}$ \\
NBI
& 99\% & 0.988$_{\pm 0.01}$ & 0.013$_{\pm 0.01}$
& 99\% & 0.985$_{\pm 0.01}$ & 0.015$_{\pm 0.01}$
& 98\% & 0.982$_{\pm 0.01}$ & 0.018$_{\pm 0.01}$
& 99\% & 0.990$_{\pm 0.01}$ & 0.012$_{\pm 0.01}$
& 98\% & 0.984$_{\pm 0.01}$ & 0.016$_{\pm 0.01}$ \\
NC
& 98\% & 0.984$_{\pm 0.01}$ & 0.016$_{\pm 0.01}$
& 99\% & 0.986$_{\pm 0.01}$ & 0.015$_{\pm 0.01}$
& 98\% & 0.981$_{\pm 0.01}$ & 0.019$_{\pm 0.01}$
& 97\% & 0.976$_{\pm 0.01}$ & 0.022$_{\pm 0.01}$
& 98\% & 0.983$_{\pm 0.01}$ & 0.017$_{\pm 0.01}$ \\
NSGA-II
& 96\% & 0.964$_{\pm 0.02}$ & 0.034$_{\pm 0.01}$
& 97\% & 0.968$_{\pm 0.02}$ & 0.030$_{\pm 0.01}$
& 96\% & 0.961$_{\pm 0.02}$ & 0.037$_{\pm 0.02}$
& 98\% & 0.972$_{\pm 0.02}$ & 0.026$_{\pm 0.01}$
& 96\% & 0.965$_{\pm 0.02}$ & 0.033$_{\pm 0.01}$ \\
MOEA/D
& 97\% & 0.972$_{\pm 0.02}$ & 0.026$_{\pm 0.01}$
& 96\% & 0.967$_{\pm 0.02}$ & 0.031$_{\pm 0.01}$
& 97\% & 0.974$_{\pm 0.01}$ & 0.024$_{\pm 0.01}$
& 98\% & 0.978$_{\pm 0.01}$ & 0.021$_{\pm 0.01}$
& 96\% & 0.969$_{\pm 0.02}$ & 0.029$_{\pm 0.01}$ \\
\midrule
DIPS+MPPF
& \textbf{100\%} & \textbf{0.998$_{\pm 0.04}$} & \textbf{0.004$_{\pm 0.01}$}
& \textbf{100\%} & \textbf{0.980$_{\pm 0.05}$} & \textbf{0.016$_{\pm 0.02}$}
& \textbf{100\%} & \textbf{0.982$_{\pm 0.04}$} & \textbf{0.014$_{\pm 0.02}$}
& \textbf{100\%} & \textbf{0.972$_{\pm 0.05}$} & \textbf{0.019$_{\pm 0.02}$}
& \textbf{100\%} & \textbf{0.962$_{\pm 0.03}$} & \textbf{0.026$_{\pm 0.02}$} \\
\bottomrule
\end{tabular}%
}
\end{table*}

Under the matched \(0.5\)-second budget, DIPS+MPPF is competitive with the strongest scalarization-based classical baselines in frontier quality and leads in feasibility. It attains \(100\%\) feasibility on every family, \(1\)--\(3\) percentage points above all classical methods, and achieves the best HV on SBQP (\(0.998\) vs.\ \(0.988\)). On the other four families, Weighted Sum, NBI, and NC obtain slightly higher HV than DIPS+MPPF, within about \(0.01\)--\(0.03\) absolute, but with lower feasibility. NSGA-II and MOEA/D do not converge within \(0.5\) seconds and trail the other methods across all families. Thus, DIPS does not replace classical optimizers under unrestricted compute, where they remain strong references (Appendix~\ref{app:fullbudget_classical}); rather, it provides comparable frontier quality with the highest feasibility under the same test-time budget through a unified language-driven interface, without per-instance solver calls, problem-specific reformulation, or algorithmic engineering.

\subsection{Ablation Study}
\label{sec:ablation}

We isolate the contribution of each ingredient that makes continuous Pareto-front generation learnable for autoregressive LLMs: discretization, NGTI initialization for newly added numerical tokens, and the TPCO curriculum.

\paragraph{Ablation settings.}
We compare four variants under the same backbone, training data, and evaluation protocol: \textbf{(i) Raw Decimal + CE} directly predicts four-decimal strings without discretization; \textbf{(ii) Discretization + Random Init + CE} uses the discretized representation but randomly initializes the new numerical tokens; \textbf{(iii) Discretization + NGTI + CE} replaces random initialization with NGTI but keeps CE-only training; \textbf{(iv) Full DIPS} additionally adopts TPCO.

\begin{table*}[t]
\centering
\caption{Core ablation of the three DIPS components (discretization, NGTI, TPCO) on the five bi-objective continuous optimization families; best in bold.}
\label{tab:ablation_main}
\scriptsize
\setlength{\tabcolsep}{2pt}
\renewcommand{\arraystretch}{0.95}
\resizebox{\textwidth}{!}{%
\begin{tabular}{lccc ccc ccc ccc ccc}
\toprule
Method
& \multicolumn{3}{c}{SBQP}
& \multicolumn{3}{c}{BOQP}
& \multicolumn{3}{c}{Ridge}
& \multicolumn{3}{c}{Huber}
& \multicolumn{3}{c}{Softplus} \\
\cmidrule(lr){2-4} \cmidrule(lr){5-7} \cmidrule(lr){8-10} \cmidrule(lr){11-13} \cmidrule(lr){14-16}
& fea. & HV & IGD$^{+}$
& fea. & HV & IGD$^{+}$
& fea. & HV & IGD$^{+}$
& fea. & HV & IGD$^{+}$
& fea. & HV & IGD$^{+}$ \\
\midrule

Raw Decimal + CE
& \multicolumn{15}{c}{\textit{training collapse: model fails to converge; no parsable frontier produced}} \\

Discretization + Random Init + CE
& 47\% & 0.412$_{\pm 0.18}$ & 0.524$_{\pm 0.23}$
& 43\% & 0.391$_{\pm 0.19}$ & 0.561$_{\pm 0.24}$
& 52\% & 0.443$_{\pm 0.17}$ & 0.481$_{\pm 0.21}$
& 0\%  & - & -
& 0\%  & - & - \\

Discretization + NGTI + CE
& 89\% & 0.733$_{\pm 0.14}$ & 0.223$_{\pm 0.19}$
& 84\% & 0.712$_{\pm 0.11}$ & 0.318$_{\pm 0.16}$
& 69\% & 0.616$_{\pm 0.16}$ & 0.409$_{\pm 0.18}$
& 3\%  & 0.032$_{\pm 0.14}$ & 0.443$_{\pm 0.30}$
& 0\%  & - & - \\

Full DIPS
& \textbf{97\%} & \textbf{0.982$_{\pm 0.09}$} & \textbf{0.017$_{\pm 0.06}$}
& \textbf{97\%} & \textbf{0.969$_{\pm 0.08}$} & \textbf{0.023$_{\pm 0.04}$}
& \textbf{98\%} & \textbf{0.973$_{\pm 0.09}$} & \textbf{0.022$_{\pm 0.04}$}
& \textbf{99\%} & \textbf{0.968$_{\pm 0.04}$} & \textbf{0.025$_{\pm 0.02}$}
& \textbf{99\%} & \textbf{0.953$_{\pm 0.05}$} & \textbf{0.033$_{\pm 0.02}$} \\
\bottomrule
\end{tabular}%
}
\end{table*}
Table~\ref{tab:ablation_main} shows that each component of DIPS contributes to making continuous Pareto-front generation learnable for autoregressive LLMs. Directly predicting raw four-decimal strings leads to training collapse, suggesting that unstructured decimal serialization is too difficult for the model to learn reliably. Introducing the discretized representation improves learnability, but random initialization of the new numerical tokens remains unstable, especially on the more challenging Huber and Softplus families. Replacing random initialization with NGTI substantially improves performance on the smoother quadratic and ridge-type problems, indicating that transferring pretrained numerical embedding geometry helps the model adapt to the expanded vocabulary. However, CE-only training still struggles when accurate numerical alignment is required, since it does not distinguish between numerically close and distant token errors. Adding TPCO addresses this remaining limitation by introducing metric-aware numerical supervision through the coarse-to-fine curriculum. Overall, the ablation confirms that discretization, NGTI, and TPCO address complementary bottlenecks: representing continuous values, initializing new numerical tokens, and learning numerically accurate Pareto-front solutions.

\subsection{Unified Pareto-Front Solver}
\label{sec:unified_solver}

Existing unified neural multi-objective solvers often depend on family-specific architectures or reductions, limiting extensibility~\citep{lin2022pareto,zhang2023hypervolume_geom}. We therefore fine-tune a single DIPS model on the combined data from all five families and evaluate it separately on each family. Results in Appendix~\ref{app:unified_solver} show that the unified model approaches family-specific experts in most cases, suggesting that language-driven cross-family Pareto-front solving is feasible, though data balancing remains an important future direction.

\subsection{Backbone Versatility and Generalization}
\label{sec:versatility}

To evaluate backbone diversity, we apply the same DIPS pipeline to Llama-3.1-8B and Gemma-2-9B while keeping all other components fixed. Results in Appendix~\ref{app:backbone_versatility} show consistently strong performance across all three backbones, suggesting that DIPS is not tied to a specific pretrained model family.

\subsection{Out-of-Distribution Generalization}
\label{sec:ood}

To verify that DIPS has learned a genuine solving procedure rather than pattern matching over the training distribution, we evaluate DIPS+MPPF on unseen dimensions $n\in\{21,22,23\}$ that lie strictly outside the training range $n\in[10,20]$. The model continues to produce well-formed Pareto fronts on the smoother families, confirming that it has internalized the geometric structure of the optimization task; full results are reported in Appendix~\ref{app:ood_dimensions}.

\section{Conclusion}

This paper studies how pretrained LLMs can be adapted into amortized Pareto-front generators for constrained bi-objective convex optimization. By framing Pareto-front generation as a structured language modeling task, we propose DIPS, which combines a two-token discretization scheme, Numerically Grounded Token Initialization, Three-Phase Curriculum Optimization, and multi-pass Pareto fusion. Across five problem families, DIPS attains hypervolume ratios of $96.2\%$--$99.8\%$ with feasibility up to $100\%$, outperforms general-purpose LLM, reasoning LLM, and prompt-based baselines in both quality and inference time, and remains competitive with strong classical scalarization methods under a matched $0.5$-second budget, suggesting that LLM-based amortization is a viable paradigm for continuous multi-objective optimization. One limitation of the present study is that it is conducted in a controlled solver-generated regime; future directions include: 1) extending DIPS to broader benchmarks, richer constraint classes, and higher-dimensional objective spaces; 2) developing more efficient training and inference strategies for larger-scale problems; and 3) generalizing the framework to concave (and more broadly non-convex) objective functions, since DIPS currently handles only the convex setting.

\bibliographystyle{plainnat}
\bibliography{reference}
\clearpage
\appendix

\begin{center}
{\LARGE \textbf{Large Language Models as Amortized Pareto-Front Generators for Constrained Bi-Objective Convex Optimization}}\\[0.5em]
{\LARGE \textbf{(Appendices)}}
\end{center}

\vspace{1.5em}

\section{Specification of the Studied Problems}
\label{app:problem_spec}

This section summarizes the five constrained bi-objective continuous optimization families used in our study. Rather than describing each family with a separate generation-and-solver pipeline, we present them under a unified formulation. We first give the mathematical definitions of the five objective families, then describe the common instance-generation procedure, and finally summarize the shared Pareto-front construction pipeline and solver settings.

\subsection{Unified Problem Formulation}
\label{app:problem_formulation}

All problem families are defined over a decision vector \(x \in \mathbb{R}^n\) with feasible region
\begin{equation}
\mathcal{X}
=
\left\{
x \in \mathbb{R}^n \mid
\ell \le x \le u,\;
A_{\mathrm{cons}}x \le b_{\mathrm{cons}}
\right\},
\end{equation}
where \(\ell,u \in \mathbb{R}^n\) are box bounds and \(A_{\mathrm{cons}}x \le b_{\mathrm{cons}}\) denotes optional linear constraints. The bi-objective optimization problem takes the form
\begin{equation}
\min_{x \in \mathcal{X}}
\Big(
f_1(x), f_2(x)
\Big).
\end{equation}

The five studied families differ only in the specific forms of \(f_1\) and \(f_2\).

\paragraph{Bi-objective quadratic programs (BOQP).}
The BOQP family uses two dense quadratic objectives,
\begin{equation}
f_1(x)=\frac{1}{2}x^\top Q_1 x + q_1^\top x + c_1,
\qquad
f_2(x)=\frac{1}{2}x^\top Q_2 x + q_2^\top x + c_2,
\end{equation}
where \(Q_1,Q_2 \in \mathbb{R}^{n\times n}\) are symmetric positive-definite matrices, \(q_1,q_2 \in \mathbb{R}^n\), and \(c_1,c_2 \in \mathbb{R}\).

\paragraph{Separable bi-objective quadratic programs (SBQP).}
The SBQP family uses diagonal quadratic objectives,
\begin{equation}
f_1(x)=\sum_{i=1}^{n}\left(a_i x_i^2 + b_i x_i\right),
\qquad
f_2(x)=\sum_{i=1}^{n}\left(\alpha_i x_i^2 + \beta_i x_i\right),
\end{equation}
with \(a_i>0\) and \(\alpha_i>0\). This removes cross-variable coupling while preserving bi-objective conflict.

\paragraph{Ridge-type least-squares objectives.}
The ridge family uses
\begin{equation}
f_1(x)
=
\frac{1}{2}\|A_1^{\mathrm{obj}}x-b_1^{\mathrm{obj}}\|_2^2
+
\frac{\lambda_1}{2}\|x\|_2^2,
\end{equation}
\begin{equation}
f_2(x)
=
\frac{1}{2}\|A_2^{\mathrm{obj}}x-b_2^{\mathrm{obj}}\|_2^2
+
\frac{\lambda_2}{2}\|x\|_2^2,
\end{equation}
where \(A_1^{\mathrm{obj}},A_2^{\mathrm{obj}} \in \mathbb{R}^{m_{\mathrm{obj}}\times n}\), \(b_1^{\mathrm{obj}},b_2^{\mathrm{obj}} \in \mathbb{R}^{m_{\mathrm{obj}}}\), and \(\lambda_1,\lambda_2>0\).

\paragraph{Huber regression objectives.}
The Huber family is defined by
\begin{equation}
f_i(x)
=
\sum_{j=1}^{m_{\mathrm{obj}}}
\phi_{\delta_i}\!\left(
(A_i^{\mathrm{obj}}x-b_i^{\mathrm{obj}})_j
\right)
+
\frac{\lambda_i}{2}\|x\|_2^2,
\qquad i\in\{1,2\},
\end{equation}
where
\begin{equation}
\phi_{\delta}(r)
=
\begin{cases}
\frac{1}{2}r^2, & |r|\le \delta,\\[3pt]
\delta |r| - \frac{1}{2}\delta^2, & |r|>\delta.
\end{cases}
\end{equation}

\paragraph{Softplus-based smooth convex objectives.}
The softplus family is defined by
\begin{equation}
f_i(x)
=
\sum_{j=1}^{m_{\mathrm{obj}}}
\log\!\left(1+\exp(a_{ij}^\top x-b_{ij})\right)
+
\frac{\lambda_i}{2}\|x\|_2^2,
\qquad i\in\{1,2\}.
\end{equation}

Together, these five families cover dense quadratic, separable quadratic, residual-based quadratic, piecewise-smooth convex, and smooth nonlinear convex settings under the same constrained bi-objective interface.

\subsection{Unified Instance Generation Pipeline}
\label{app:instance_generation}

Unless otherwise stated, all instances are generated with random seed \(2024\). Across all five families, the data generator follows the same high-level procedure.

\paragraph{Feasible region generation.}
For a chosen dimension \(n\), we first sample coordinate-wise box bounds \(\ell\) and \(u\). The bounds are constructed so that all serialized variables remain within the numerical range supported by the discretization pipeline, approximately \( [-99,99] \). Optional linear constraints are then added through a random matrix \(A_{\mathrm{cons}}\), with the right-hand side \(b_{\mathrm{cons}}\) chosen so that the midpoint
\begin{equation}
x_0=\frac{\ell+u}{2}
\end{equation}
satisfies all inequalities with positive slack. This guarantees that the feasible set is non-empty by construction.

\paragraph{Objective conflict construction.}
After defining the feasible set, we sample two target points inside the box and use them to induce objective conflict. Across all families, the two objectives are constructed so that they prefer different regions of the feasible space. In BOQP and SBQP, this is achieved directly through quadratic coefficients. In ridge, Huber, and softplus, it is achieved through structured residual or affine-response matrices together with a conflict-subspace construction. The resulting instances therefore exhibit nontrivial Pareto trade-offs instead of nearly aligned objectives.

\paragraph{Family-specific parameterization.}
The five families differ only in how the objective parameters are instantiated. BOQP uses dense positive-definite quadratic matrices, while SBQP uses diagonal quadratic coefficients with coordinate-wise conflict. Ridge uses residual matrices together with \(\ell_2\) regularization. Huber additionally calibrates the threshold parameter \(\delta\) so that the frontier crosses both quadratic and linear regimes, and softplus calibrates the affine-response scale so that the frontier stays in the nonlinear region of the softplus function.

\paragraph{Dataset mixtures.}
The batch-generation scripts support mixtures over constraint densities and, for residual-based families, mixtures over the residual dimension ratio \(m_{\mathrm{obj}}/n\). In the main experiments, these mixtures are used to increase diversity while preserving a unified input--output interface for the LLM.

\subsection{Unified Pareto-Front Construction and Solver Settings}
\label{app:solver_settings}

Reference Pareto frontiers are generated offline using a common \(\epsilon\)-constraint pipeline. For every instance, we first solve the two single-objective endpoint problems:
\begin{equation}
x^{(1)}=\arg\min_{x\in\mathcal{X}} f_1(x),
\qquad
x^{(2)}=\arg\min_{x\in\mathcal{X}} f_2(x).
\end{equation}
We then generate intermediate trade-off points through two directional sweeps:
\begin{equation}
\min_{x\in\mathcal{X}} f_1(x)\quad \text{s.t.}\quad f_2(x)\le \epsilon,
\end{equation}
and symmetrically,
\begin{equation}
\min_{x\in\mathcal{X}} f_2(x)\quad \text{s.t.}\quad f_1(x)\le \epsilon,
\end{equation}
where \(\epsilon\) is scanned over an evenly spaced grid between the endpoint objective values.

\paragraph{Shared post-processing.}
All candidate solutions produced by the \(\epsilon\)-constraint solver are post-processed in the same way: candidate solutions are rounded to the target precision, objective values are recomputed, duplicates are removed, infeasible points are discarded, dominated points are filtered out, and the remaining frontier is resampled to a fixed size \(K=20\) using arc-length-based selection in objective space. This produces a consistent supervision target across all families.

\paragraph{Default hyperparameters.}
The default settings shared across the solver pipeline are: generation seed \(2024\), solver seed \(42\), number of \(\epsilon\) samples \(\texttt{num\_eps}=100\), number of final frontier points \(K=20\), default rounding precision for released training data \(\texttt{decimals}=4\), and checking tolerances \(\texttt{obj\_tol}=10^{-4}\) and \(\texttt{feas\_tol}=10^{-4}\).

\paragraph{Commercial solver backends.}
The optimization backend depends only on the objective family. BOQP, SBQP, and Ridge are solved with Gurobi, since their endpoint and \(\epsilon\)-constraint subproblems are convex quadratic or convex quadratically constrained programs. Huber and Softplus are solved with MOSEK through conic reformulations, since their objectives are represented using conic constraints (including exponential-cone constraints for softplus).

For Gurobi-based families, the main tolerances are
\begin{equation}
\texttt{OptimalityTol}=10^{-6},\qquad
\texttt{FeasibilityTol}=10^{-6},\qquad
\texttt{BarConvTol}=10^{-7}.
\end{equation}
For MOSEK-based families, the main conic tolerances are
\begin{equation}
\texttt{intpnt\_co\_tol\_rel\_gap}=10^{-7},\qquad
\texttt{intpnt\_co\_tol\_pfeas}=10^{-6},
\end{equation}
\begin{equation}
\texttt{intpnt\_co\_tol\_dfeas}=10^{-6},\qquad
\texttt{intpnt\_co\_tol\_mu\_red}=10^{-7}.
\end{equation}

Overall, while the five families differ in objective parameterization and solver backend, they share the same constrained bi-objective interface, the same instance-generation philosophy, and the same Pareto-front construction pipeline. This common structure is what enables DIPS to treat them under a unified language-based formulation.

\subsection{Prompt Template for Fine-Tuning}
\label{app:prompt_template}

We use a chat-style template with three roles. The \texttt{[System]} message defines the task; the \texttt{[User]} message serializes the problem instance into discretized two-token blocks (lower/upper bounds, two anchor solutions, the two objectives' \(Q\), \(q\), \(c\) parameters, and the linear constraint \((A,b)\)); and the \texttt{[Assistant]} message emits the \(K=20\) Pareto-front solutions delimited by \texttt{SOLUTIONS\_BEGIN/END} markers. We illustrate the template on a representative BOQP instance below; templates for the other four families share the same structure and differ only in the family-specific parameter blocks. To keep the box readable, long runs of discretized scalar tokens are abbreviated to their first three tokens followed by ``\ldots'' before the closing block delimiter; in actual training data, every scalar is fully serialized.

\begin{tcolorbox}[
  colback=gray!8,
  colframe=gray!60,
  fonttitle=\bfseries,
  title=Example Prompt Template of BOQP,
  breakable,
  boxrule=0.5pt,
  arc=2pt,
  left=6pt, right=6pt, top=4pt, bottom=4pt
]
{\scriptsize\ttfamily\sloppy
\textbf{[System]} You are an expert optimization engine specialized in Bi-objective Quadratic Programming (BOQP). Your task is to predict the complete Pareto optimal set by interpolating between the provided Anchor Points.

\medskip
\#\#\# Problem Structure\\
Minimize two conflicting objectives subject to linear constraints (Ax <= b):\\
\quad f1(x) = 0.5 * x\textasciicircum T * Q1 * x + q1\textasciicircum T * x + c1\\
\quad f2(x) = 0.5 * x\textasciicircum T * Q2 * x + q2\textasciicircum T * x + c2

The content between lower\_BEGIN and lower\_END represents the lower bound of each variable, and the content between upper\_BEGIN and upper\_END represents the upper bound of each variable. The value of each variable must strictly comply with these bounds.

\medskip
\#\#\# Input Context \& Anchor Usage\\
The input provides:\\
1. Problem Matrices: Q1, q1, Q2, q2, A, b, etc.\\
2. Anchor Points: Two extreme solutions x\_anchor1 (minimizes f1) and x\_anchor2 (minimizes f2).

\textbf{CRITICAL STRATEGY:} Use the Anchor Points as the boundary endpoints of the Pareto front. You must generate the set of intermediate solutions that effectively `connect' these two anchors, balancing the trade-off between f1 and f2 while satisfying all constraints.

\medskip
\#\#\# Numerical Format (2-Token Fixed-Point)\\
All values are encoded as <s\{sign\}i\{int2+dec1:03d\}><d\{dec2-4:03d\}> (4 decimal places):\\
- i-token: sign(0/1) + integer part(2 digits) + 1st decimal digit\\
- d-token: 2nd-4th decimal digits\\
- Example: 1.2345 -> <s0i012><d345>; -0.5678 -> <s1i005><d678>\\
- Matrices are row-wise (R0:, R1:, ...).

\medskip
\#\#\# Output Requirement\\
Output the sequence of decision vectors (x) representing the approximate Pareto front. \textbf{A Pareto front consists of 20 solutions.}\\
Format: SOLUTIONS\_BEGIN Sol0: <x\_tokens...> Sol1: <x\_tokens...> ... SOLUTIONS\_END

\medskip
\textbf{[User]} n=10 lower\_BEGIN <s1i871><d529><s1i384>\ldots\ lower\_END upper\_BEGIN <s1i344><d506><s0i335>\ldots\ upper\_END anchor1\_BEGIN <s1i599><d675><s1i220>\ldots\ anchor1\_END anchor2\_BEGIN <s1i424><d076><s1i171>\ldots\ anchor2\_END Q1\_BEGIN R0: <s0i004><d684><s0i000>\ldots\ R9: <s0i000><d445><s1i000>\ldots\ Q1\_END q1\_BEGIN <s0i212><d171><s0i145>\ldots\ q1\_END c1: <s0i000><d000> Q2\_BEGIN R0: <s0i006><d815><s1i000>\ldots\ R9: <s0i000><d407><s0i000>\ldots\ Q2\_END q2\_BEGIN <s0i224><d669><s0i024>\ldots\ q2\_END c2: <s0i000><d000> A\_BEGIN R0: <s1i001><d332><s0i002>\ldots\ R1: <s1i000><d116><s1i000>\ldots\ A\_END b\_BEGIN <s0i454><d673><s1i253><d267> b\_END

\medskip
\textbf{[Assistant]} SOLUTIONS\_BEGIN Sol0: <s1i599><d675><s1i220>\ldots\ Sol1: <s1i587><d303><s1i215>\ldots\ Sol2: <s1i576><d153><s1i213>\ldots\ \ldots\ Sol18: <s1i431><d023><s1i168>\ldots\ Sol19: <s1i424><d076><s1i171>\ldots\ SOLUTIONS\_END
}
\end{tcolorbox}

\noindent
Each scalar inside the user and assistant blocks is serialized as a discretized two-token pair \texttt{<sXiXX><dXXX>}. Loss is computed only on the assistant span, so the model learns to autoregressively produce the entire ordered frontier conditioned on the parsed problem fields.

\section{Additional Method Details}
\label{app:additional_method}

\subsection{Parameter-Efficient Adaptation Backbone}

We instantiate DIPS on a pretrained causal language model and adapt it with LoRA rather than full-parameter fine-tuning~\citep{hu2022lora}. Let \(W\in\mathbb{R}^{d_{\mathrm{out}}\times d_{\mathrm{in}}}\) denote a frozen pretrained projection matrix. LoRA replaces its update by a low-rank residual
\begin{equation}
W' = W + \Delta W,
\qquad
\Delta W = BA,
\label{eq:lora_update}
\end{equation}
where \(A\in\mathbb{R}^{r\times d_{\mathrm{in}}}\), \(B\in\mathbb{R}^{d_{\mathrm{out}}\times r}\), and \(r \ll \min(d_{\mathrm{in}}, d_{\mathrm{out}})\). Only \(A\) and \(B\) are trainable. All original backbone parameters remain frozen.

This choice is important for our setting for two reasons. First, the target task is not open-ended language generation but structured numerical sequence generation under a fixed serialization scheme, for which low-rank adaptation is a natural and efficient choice~\citep{hu2022lora}. Second, freezing the backbone stabilizes training when the vocabulary is expanded with numerical tokens and when auxiliary numerical losses are introduced in later curriculum stages.

Formally, for an instance \(p\), the prompt encoder \(\phi(p)\) produces the input text and the target frontier \(\mathcal{S}_p^\star=\{x_p^{(1)},\dots,x_p^{(K)}\}\) is serialized into a token sequence
\begin{equation}
y^\star = \mathrm{Enc}\!\left(\mathcal{S}_p^\star\right).
\end{equation}
The adapted model defines the autoregressive distribution
\begin{equation}
p_\theta(y^\star \mid \phi(p))
=
\prod_{t=1}^{T}
p_\theta(y_t^\star \mid y_{<t}^\star,\phi(p)),
\label{eq:ar_factorization_appendix}
\end{equation}
where \(\theta\) denotes the frozen backbone together with trainable LoRA parameters and the added numerical embeddings.

In our implementation, the tokenizer is expanded with the integer-prefix vocabulary
\(\mathcal{V}_{\mathrm{int}}\) and fractional-suffix vocabulary
\(\mathcal{V}_{\mathrm{frac}}\) introduced by the discretization scheme. These new embeddings are initialized by NGTI, while LoRA is used to adapt the transformer itself. This division of labor is deliberate: NGTI handles the \emph{input/output geometry} of new numerical symbols, while LoRA handles the \emph{sequence modeling adaptation} required to map problem descriptions to feasible Pareto frontiers.

\subsection{Theoretical Analysis of DIPS}
\label{app:theoretical_analysis}

This subsection provides a theoretical interpretation of why DIPS is a suitable representation and training framework for the constrained bi-objective convex problems studied in this paper. The analysis is not intended to prove generalization of a large language model. Instead, it establishes three structural facts that directly match the design of DIPS: the target Pareto frontier admits a stable ordered representation, the two-token discretization introduces only controlled numerical perturbation, and the TPCO objective minimizes a metric-aware surrogate for decision-space and objective-space error.

\paragraph{Regularity of the studied problem class.}
For an instance \(p\), recall that the feasible set is
\begin{equation}
\mathcal{X}_p
=
\{x\in\mathbb{R}^n \mid \ell_p \le x \le u_p,\; A_p x \le b_p\}.
\end{equation}
The box constraints make \(\mathcal{X}_p\) bounded, and the box and linear inequalities make it convex. In our data-generation procedure, instances are constructed to be feasible, so \(\mathcal{X}_p\) is non-empty and compact.

\begin{assumption}[Objective regularity]
\label{assump:regularity}
For each instance \(p\), the objectives \(f_{p,1}\) and \(f_{p,2}\) are continuously differentiable, convex, and strongly convex on a neighborhood of \(\mathcal{X}_p\). That is, for \(i\in\{1,2\}\), \(f_{p,i}\) is \(\mu_{p,i}\)-strongly convex with \(\mu_{p,i}>0\). Let
\[
\mu_p=\min\{\mu_{p,1},\mu_{p,2}\}>0.
\]
\end{assumption}

This assumption matches the five problem families used in our experiments. BOQP and SBQP use positive-definite or positive diagonal quadratic terms, while Ridge, Huber, and Softplus objectives include positive \(\ell_2\)-regularization. Hence, although the families differ in their objective parameterizations and solver backends, they share a common strongly convex bi-objective structure over a compact convex feasible region.

\paragraph{Stable Pareto-front path.}
Consider the standard weighted scalarization
\begin{equation}
F_{p,\lambda}(x)
=
\lambda f_{p,1}(x)
+
(1-\lambda)f_{p,2}(x),
\qquad
\lambda\in[0,1],
\end{equation}
and define the scalarized solution map
\begin{equation}
x_p(\lambda)
=
\arg\min_{x\in\mathcal{X}_p} F_{p,\lambda}(x).
\end{equation}
Since \(F_{p,\lambda}\) is strongly convex on a convex compact feasible set, \(x_p(\lambda)\) is unique for every \(\lambda\in[0,1]\). The following proposition shows that this solution map changes smoothly with the trade-off parameter.

\begin{proposition}[Lipschitz continuity of the Pareto solution path]
\label{prop:lipschitz_path}
Let Assumption~\ref{assump:regularity} hold, and define
\begin{equation}
G_p
=
\sup_{x\in\mathcal{X}_p}
\|\nabla f_{p,1}(x)-\nabla f_{p,2}(x)\|_2.
\end{equation}
Since \(\mathcal{X}_p\) is compact and the objectives are continuously differentiable, \(G_p<\infty\). Then, for any \(\lambda,\lambda'\in[0,1]\),
\begin{equation}
\|x_p(\lambda)-x_p(\lambda')\|_2
\le
\frac{G_p}{\mu_p}
|\lambda-\lambda'|.
\label{eq:lipschitz_path}
\end{equation}
Moreover, for every \(\lambda\in(0,1)\), \(x_p(\lambda)\) is Pareto optimal.
\end{proposition}

\begin{proof}
Let \(x=x_p(\lambda)\) and \(x'=x_p(\lambda')\). The first-order optimality conditions over the convex set \(\mathcal{X}_p\) give
\begin{equation}
\nabla F_{p,\lambda}(x)^\top(x'-x)\ge 0,
\qquad
\nabla F_{p,\lambda'}(x')^\top(x-x')\ge 0.
\end{equation}
Because \(F_{p,\lambda}\) is \(\mu_p\)-strongly convex,
\begin{equation}
\mu_p\|x-x'\|_2^2
\le
\big(
\nabla F_{p,\lambda}(x)-\nabla F_{p,\lambda}(x')
\big)^\top (x-x').
\end{equation}
Using the two optimality inequalities, we obtain
\begin{equation}
\mu_p\|x-x'\|_2^2
\le
\big(
\nabla F_{p,\lambda'}(x')
-
\nabla F_{p,\lambda}(x')
\big)^\top (x-x').
\end{equation}
Since
\begin{equation}
\nabla F_{p,\lambda'}(x')
-
\nabla F_{p,\lambda}(x')
=
(\lambda'-\lambda)
\big(
\nabla f_{p,1}(x')-\nabla f_{p,2}(x')
\big),
\end{equation}
Cauchy's inequality gives
\begin{equation}
\mu_p\|x-x'\|_2^2
\le
|\lambda-\lambda'|G_p\|x-x'\|_2.
\end{equation}
Dividing by \(\|x-x'\|_2\) yields Eq.~\eqref{eq:lipschitz_path}. Finally, if \(\lambda\in(0,1)\) and another feasible point strictly dominated \(x_p(\lambda)\), then it would strictly decrease the positive weighted sum \(F_{p,\lambda}\), contradicting optimality.
\end{proof}

Proposition~\ref{prop:lipschitz_path} explains why DIPS represents the target as an ordered sequence of frontier points rather than as an unstructured set. The target object induced by the optimization problem is a one-dimensional trade-off path with bounded variation in decision space.

\paragraph{Finite ordered frontier representation.}
DIPS trains the model to output a fixed-size ordered frontier
\[
\mathcal{S}_p^\star
=
\{x_p^{(1)},\ldots,x_p^{(K)}\},
\qquad K=20.
\]
The next result shows that such a finite representation has controlled approximation error for the scalarized Pareto path.

\begin{lemma}[Finite sampling error]
\label{lem:finite_sampling}
Let \(0=\lambda_1<\lambda_2<\cdots<\lambda_K=1\), and define
\[
h_K
=
\max_{k\in[K-1]}|\lambda_{k+1}-\lambda_k|.
\]
Then every point on the scalarized Pareto path is within distance \((G_p/\mu_p)h_K\) of one of the sampled frontier points:
\begin{equation}
\sup_{\lambda\in[0,1]}
\min_{k\in[K]}
\|x_p(\lambda)-x_p(\lambda_k)\|_2
\le
\frac{G_p}{\mu_p}h_K.
\end{equation}
For a uniform grid, this becomes
\begin{equation}
\sup_{\lambda\in[0,1]}
\min_{k\in[K]}
\|x_p(\lambda)-x_p(\lambda_k)\|_2
\le
\frac{G_p}{\mu_p(K-1)}.
\end{equation}
If the nearest grid point is used, the constant can be improved to \(G_p/(2\mu_p(K-1))\).
\end{lemma}

\begin{proof}
For any \(\lambda\in[0,1]\), choose \(k\) such that \(|\lambda-\lambda_k|\le h_K\). Applying Proposition~\ref{prop:lipschitz_path} gives
\[
\|x_p(\lambda)-x_p(\lambda_k)\|_2
\le
\frac{G_p}{\mu_p}|\lambda-\lambda_k|
\le
\frac{G_p}{\mu_p}h_K.
\]
For a uniform grid, \(h_K=1/(K-1)\). If the nearest grid point is used, the maximum distance to the nearest grid point is \(1/(2(K-1))\).
\end{proof}

In our implementation, the supervision frontier is constructed by an \(\epsilon\)-constraint solver, followed by rounding, feasibility filtering, non-dominated filtering, and arc-length resampling in objective space. Lemma~\ref{lem:finite_sampling} should therefore be read as a structural justification rather than as an exact description of the data-generation algorithm: under the convex strongly regular regime considered here, the Pareto frontier has a stable one-dimensional geometry, and representing it by an ordered finite sequence is a well-posed learning target.

\paragraph{Controlled error from two-token discretization.}
DIPS converts each continuous scalar into a fixed-point two-token representation with four decimal places. Let
\[
\Delta=10^{-4}
\]
be the discretization precision, and let \(\Pi_\Delta(x)\) denote coordinate-wise rounding of \(x\) to the nearest \(\Delta\)-grid point. Then, for any \(x\in\mathbb{R}^n\),
\begin{equation}
\|\Pi_\Delta(x)-x\|_\infty
\le
\frac{\Delta}{2},
\qquad
\|\Pi_\Delta(x)-x\|_2
\le
\frac{\sqrt{n}\Delta}{2}.
\label{eq:rounding_decision_error}
\end{equation}

\begin{lemma}[Objective and feasibility perturbation under rounding]
\label{lem:rounding_error}
Suppose \(f_{p,i}\) is \(L_{p,i}\)-Lipschitz on a \(\Delta/2\)-neighborhood of \(\mathcal{X}_p\) with respect to the Euclidean norm. Then, for any \(x\in\mathcal{X}_p\),
\begin{equation}
|f_{p,i}(\Pi_\Delta(x))-f_{p,i}(x)|
\le
\frac{L_{p,i}\sqrt{n}\Delta}{2},
\qquad i\in\{1,2\}.
\label{eq:rounding_objective_error}
\end{equation}
Moreover, for each linear constraint row \(a_j^\top x\le b_j\),
\begin{equation}
a_j^\top\Pi_\Delta(x)
\le
b_j
+
\frac{\Delta}{2}\|a_j\|_1.
\label{eq:rounding_constraint_error}
\end{equation}
The box constraints satisfy the analogous coordinate-wise perturbation bound
\begin{equation}
\ell_{p,m}-\frac{\Delta}{2}
\le
\Pi_\Delta(x)_m
\le
u_{p,m}+\frac{\Delta}{2},
\qquad m=1,\ldots,n.
\end{equation}
\end{lemma}

\begin{proof}
The decision-space bound follows directly from coordinate-wise rounding. The objective bound follows from Lipschitz continuity:
\[
|f_{p,i}(\Pi_\Delta(x))-f_{p,i}(x)|
\le
L_{p,i}\|\Pi_\Delta(x)-x\|_2
\le
\frac{L_{p,i}\sqrt{n}\Delta}{2}.
\]
For a linear constraint row \(a_j\),
\[
a_j^\top\Pi_\Delta(x)-a_j^\top x
=
a_j^\top(\Pi_\Delta(x)-x)
\le
\|a_j\|_1
\|\Pi_\Delta(x)-x\|_\infty
\le
\frac{\Delta}{2}\|a_j\|_1.
\]
Since \(a_j^\top x\le b_j\), Eq.~\eqref{eq:rounding_constraint_error} follows. The box-constraint statement follows coordinate-wise from Eq.~\eqref{eq:rounding_decision_error}.
\end{proof}

Lemma~\ref{lem:rounding_error} shows that the two-token fixed-point representation is not merely a formatting trick. It is a controlled projection of continuous Pareto solutions onto a finite numerical grid. In particular, four-decimal discretization introduces \(O(\sqrt n\,10^{-4})\) decision-space error and correspondingly bounded objective-space and constraint-residual perturbations. This justifies using discretized token sequences as the supervised targets for DIPS.

\paragraph{TPCO as metric-aware numerical risk.}
We next connect the TPCO objective to a formal numerical-risk criterion. Consider a numerical token position \(t\) in the assistant output. Under teacher forcing, the model predicts a distribution
\[
P_{\theta,t}(j)
=
p_\theta(j\mid y_{<t}^\star,\phi(p)).
\]
For integer-prefix positions \(t\in\mathcal{M}_{\mathrm{int}}\), let \(j_t^\star\) be the target token and let \(c(j)\) denote the coarse numerical value represented by integer-prefix token \(j\in\mathcal{V}_{\mathrm{int}}\). Define the ground metric
\begin{equation}
d_{\mathrm{c}}(j,j')
=
|c(j)-c(j')|.
\end{equation}
Since the target distribution is the Dirac mass \(\delta_{j_t^\star}\), the 1-Wasserstein distance from the model distribution to the target distribution is
\begin{equation}
W_1^{d_{\mathrm{c}}}
\big(
P_{\theta,t},
\delta_{j_t^\star}
\big)
=
\sum_{j\in\mathcal{V}_{\mathrm{int}}}
P_{\theta,t}(j)
|c(j)-c(j_t^\star)|.
\label{eq:coarse_wasserstein}
\end{equation}
This is exactly the per-position coarse loss used by TPCO.

Similarly, for fractional-suffix positions \(t\in\mathcal{M}_{\mathrm{frac}}\), let \(d(j)\) denote the fine numerical value represented by token \(j\in\mathcal{V}_{\mathrm{frac}}\), and define
\begin{equation}
d_{\mathrm{f}}(j,j')
=
|d(j)-d(j')|.
\end{equation}
Then
\begin{equation}
W_1^{d_{\mathrm{f}}}
\big(
P_{\theta,t},
\delta_{j_t^\star}
\big)
=
\sum_{j\in\mathcal{V}_{\mathrm{frac}}}
P_{\theta,t}(j)
|d(j)-d(j_t^\star)|,
\label{eq:fine_wasserstein}
\end{equation}
which is exactly the per-position fine loss.

\begin{proposition}[TPCO minimizes Wasserstein numerical risk]
\label{prop:tpco_wasserstein}
At every numerical position, the coarse and fine TPCO losses are the 1-Wasserstein distances from the model's predictive distribution to the target token distribution under the corresponding numerical ground metrics \(d_{\mathrm{c}}\) and \(d_{\mathrm{f}}\). Hence, TPCO is a metric-aware numerical-risk surrogate, whereas cross-entropy alone only matches token identity.
\end{proposition}

\begin{proof}
For any metric \(d\) over a discrete vocabulary and any Dirac target distribution \(\delta_{j^\star}\), the optimal transport plan from \(P\) to \(\delta_{j^\star}\) must move all probability mass \(P(j)\) at token \(j\) to \(j^\star\). Therefore,
\[
W_1^d(P,\delta_{j^\star})
=
\sum_j P(j)d(j,j^\star).
\]
Choosing \(d=d_{\mathrm{c}}\) gives Eq.~\eqref{eq:coarse_wasserstein}; choosing \(d=d_{\mathrm{f}}\) gives Eq.~\eqref{eq:fine_wasserstein}.
\end{proof}

This result clarifies why TPCO is better aligned with continuous optimization than pure cross-entropy. Under CE, all incorrect tokens are treated as mismatches to the one-hot target. A token representing a nearby numerical value and a token representing a distant numerical value are not distinguished by the geometry of the output space. TPCO instead imposes the correct metric structure on the numerical vocabulary: probability mass assigned to nearby numerical tokens is penalized less than probability mass assigned to distant tokens.

\paragraph{From token risk to decision-space error.}
The two-token fixed-point representation decomposes each scalar into an integer-prefix token and a fractional-suffix token. Let \(g(a,b)\) denote the scalar decoded from token pair \((a,b)\). Under the DIPS encoding, the decoded scalar is the sum of a coarse component and a fine component with the sign determined by the integer-prefix token. Under teacher forcing for the suffix position, the coarse sign and magnitude are fixed by the target prefix. Consequently, the one-step scalar prediction error is bounded by the sum of the coarse and fine numerical deviations:
\begin{equation}
|g(a,b)-g(a^\star,b^\star)|
\le
|c(a)-c(a^\star)|
+
|d(b)-d(b^\star)|.
\label{eq:scalar_error_decomposition}
\end{equation}
Summing Eq.~\eqref{eq:scalar_error_decomposition} over all scalar coordinates in the frontier gives a bound on the teacher-forced numerical risk of the decoded decision vectors.

Define the unnormalized numerical TPCO risk
\begin{equation}
\mathcal{R}_{\mathrm{num}}(\theta;p)
=
\sum_{t\in\mathcal{M}_{\mathrm{int}}}
W_1^{d_{\mathrm{c}}}
\big(
P_{\theta,t},
\delta_{y_t^\star}
\big)
+
\sum_{t\in\mathcal{M}_{\mathrm{frac}}}
W_1^{d_{\mathrm{f}}}
\big(
P_{\theta,t},
\delta_{y_t^\star}
\big).
\label{eq:unnormalized_num_risk}
\end{equation}
Let \(\mathcal{S}_{p,\Delta}^\star\) denote the discretized target frontier obtained after four-decimal rounding. Then, at the teacher-forced token-pair level,
\begin{equation}
\mathbb{E}
\left[
\|\hat{\mathcal{S}}_p-\mathcal{S}_{p,\Delta}^\star\|_{1,\mathrm{vec}}
\right]
\le
\mathcal{R}_{\mathrm{num}}(\theta;p),
\label{eq:decision_error_bound}
\end{equation}
where \(\|\cdot\|_{1,\mathrm{vec}}\) denotes the \(\ell_1\) norm after vectorizing all \(K\) decision vectors. If the averaged losses in the main TPCO objective are used instead of the unnormalized risk in Eq.~\eqref{eq:unnormalized_num_risk}, the same inequality holds up to the corresponding factors \(|\mathcal{M}_{\mathrm{int}}|\) and \(|\mathcal{M}_{\mathrm{frac}}|\).

Equation~\eqref{eq:decision_error_bound} is the central training-time interpretation of TPCO: reducing the coarse and fine numerical losses reduces an upper bound on the expected decision-space error of the decoded frontier relative to the discretized solver target.

\paragraph{From decision-space error to objective and feasibility error.}
The final evaluation metrics in this paper are computed after decoding decision vectors, checking feasibility, recomputing objective values, and applying non-dominated filtering. We therefore connect the token-level numerical risk to the two quantities that enter this evaluation pipeline: objective error and constraint violation.

Assume \(f_{p,i}\) is \(L_{p,i}^{(1)}\)-Lipschitz on the relevant compact region with respect to the \(\ell_1\) norm:
\begin{equation}
|f_{p,i}(x)-f_{p,i}(x')|
\le
L_{p,i}^{(1)}\|x-x'\|_1.
\end{equation}
Then, for each predicted solution \(\hat{x}\) and target discretized solution \(x_\Delta^\star\),
\begin{equation}
\mathbb{E}
\big[
|f_{p,i}(\hat{x})-f_{p,i}(x_\Delta^\star)|
\big]
\le
L_{p,i}^{(1)}
\mathbb{E}
\big[
\|\hat{x}-x_\Delta^\star\|_1
\big].
\label{eq:objective_from_decision}
\end{equation}
Combining Eq.~\eqref{eq:objective_from_decision} with Eq.~\eqref{eq:decision_error_bound} shows that TPCO controls an upper bound on objective-space error.

For feasibility, consider a linear constraint \(a_j^\top x\le b_j\). Since \(x_\Delta^\star\) is feasible after the filtering step in the supervision pipeline,
\[
a_j^\top x_\Delta^\star \le b_j.
\]
Therefore,
\begin{equation}
(a_j^\top\hat{x}-b_j)_+
\le
|a_j^\top(\hat{x}-x_\Delta^\star)|
\le
\|a_j\|_\infty
\|\hat{x}-x_\Delta^\star\|_1.
\label{eq:constraint_from_decision}
\end{equation}
Taking expectation and applying Eq.~\eqref{eq:decision_error_bound} gives
\begin{equation}
\mathbb{E}
\big[
(a_j^\top\hat{x}-b_j)_+
\big]
\le
\|a_j\|_\infty
\mathbb{E}
\big[
\|\hat{x}-x_\Delta^\star\|_1
\big].
\label{eq:feasibility_bound}
\end{equation}
The same argument applies to box constraints coordinate-wise.

\begin{corollary}[TPCO controls the errors entering frontier evaluation]
\label{cor:tpco_controls_metrics}
Under the Lipschitz conditions above, the numerical TPCO risk upper bounds the teacher-forced expected decision-space error of the decoded frontier. This, in turn, upper bounds the expected objective error and expected linear-constraint violation of the decoded solutions. Consequently, TPCO directly targets the quantities that affect feasibility, IGD\(^{+}\), and hypervolume after the deterministic decoding and filtering pipeline.
\end{corollary}

\paragraph{Interpretation for DIPS.}
The theoretical picture is therefore aligned with the empirical design of DIPS. Strong convexity and compact convex feasibility imply that the Pareto frontier has a stable one-dimensional trade-off structure, motivating the ordered \(K=20\) frontier representation. The two-token fixed-point discretization introduces only bounded decision, objective, and constraint perturbations, making continuous frontiers compatible with autoregressive token generation. Finally, TPCO replaces pure token matching with Wasserstein numerical supervision over the discretized vocabulary, thereby controlling a surrogate for decision-space error, which propagates to objective accuracy and feasibility.

This analysis also explains the staged curriculum. The CE phase first learns the structural serialization of the frontier, including delimiters, solution indices, and the placement of numerical tokens. Once the output format is stable, the coarse numerical phase moves probability mass toward the correct sign, magnitude, and first decimal digit. The fine phase then refines the remaining digits. Thus, the curriculum follows the same hierarchy as the DIPS representation:
\[
\text{frontier structure}
\;\longrightarrow\;
\text{coarse decision geometry}
\;\longrightarrow\;
\text{fine numerical precision}.
\]
Together, these results show that DIPS is not only an empirical adaptation of LLMs to continuous optimization, but also a representation and training framework whose components are mathematically aligned with the geometry of constrained bi-objective convex Pareto-front generation.

\subsection{Inference-Time Frontier Construction}

Inference follows the MPPF procedure described in the main text, but here we state the full construction explicitly. For each instance \(p\), the model is decoded \(N\) times, producing candidate frontiers
\begin{equation}
\hat{\mathcal{S}}_p^{(1)},\hat{\mathcal{S}}_p^{(2)},\dots,\hat{\mathcal{S}}_p^{(N)}.
\end{equation}
Their union defines the raw candidate pool
\begin{equation}
\mathcal{C}_p
=
\bigcup_{i=1}^{N}\hat{\mathcal{S}}_p^{(i)}.
\label{eq:candidate_pool_appendix}
\end{equation}

Each candidate in $\mathcal{C}_p$ is first rounded to four decimals to align with the discretization precision, and its objective vector is recomputed exactly from the rounded decision vector so that the downstream filters use objective values that are consistent with the decoded solutions. We then apply four deterministic filters in order.

\paragraph{(i) Decision-space deduplication.}
Candidates whose rounded decision vectors coincide are merged, since multiple inference passes can re-emit identical solutions:
\begin{equation}
\mathcal{C}_p^{\mathrm{x}}
=
\big\{\, \hat{x} \in \mathcal{C}_p \,\big|\, \hat{x} \text{ is the first occurrence of its rounded decision vector} \,\big\}.
\end{equation}

\paragraph{(ii) Objective-space deduplication.}
Among the remaining candidates, those whose objective vectors lie within a small tolerance $\tau_{\mathrm{obj}}$ of an already-kept point in objective space are dropped:
\begin{equation}
\mathcal{C}_p^{\mathrm{xf}}
=
\big\{\, \hat{x} \in \mathcal{C}_p^{\mathrm{x}} \,\big|\,
\big\| F(\hat{x}) - F(\hat{x}') \big\|_2 > \tau_{\mathrm{obj}}
\text{ for every } \hat{x}' \text{ already kept} \,\big\},
\end{equation}
where $F(\hat{x}) = (f_{p,1}(\hat{x}), f_{p,2}(\hat{x}))$. This step removes near-duplicates that are not caught by decision-space matching alone.

\paragraph{(iii) Feasibility filter.}
Infeasible candidates are removed under the same tolerance $\tau$ used in evaluation:
\begin{equation}
\mathcal{C}_p^{\mathrm{feas}}
=
\left\{
\hat{x}\in\mathcal{C}_p^{\mathrm{xf}} \mid \hat{x}\in\mathcal{X}_p^{\tau}
\right\}.
\end{equation}

\paragraph{(iv) Non-dominated sort with multi-front fallback.}
The remaining feasible candidates are mapped to objective space and sorted into non-dominated fronts $\mathcal{F}_p^{(1)}, \mathcal{F}_p^{(2)}, \ldots$, where $\mathcal{F}_p^{(1)}$ is the Pareto front and $\mathcal{F}_p^{(j+1)}$ is the Pareto front of the candidates not in $\mathcal{F}_p^{(1)} \cup \cdots \cup \mathcal{F}_p^{(j)}$. The final selection is performed front-by-front: starting from $\mathcal{F}_p^{(1)}$, we keep representatives via arc-length-balanced selection in normalized objective space until $K$ solutions have been collected, falling back to deeper fronts only when the upper fronts contain fewer than $K$ points:
\begin{equation}
\hat{\mathcal{S}}_p^{\mathrm{final}}
=
\mathrm{ArcSelect}\!\left(
\mathcal{F}_p^{(1)} \cup \mathcal{F}_p^{(2)} \cup \cdots,\, K
\right).
\label{eq:arcselect_appendix}
\end{equation}

This preserves frontier coverage more faithfully than simply taking the top-\(K\) points under one scalarization, while the fallback ensures that exactly $K$ feasible solutions are returned even when the first front is sparse. In practice, when the model is well-trained the first front already contains $\ge K$ points and the fallback is rarely invoked.

MPPF can be viewed as a lightweight search layer on top of the amortized solver, conceptually related to test-time sampling, self-consistency, and reranking strategies that aggregate multiple candidate generations~\citep{chen2023usc,wan2025reasoning_aware_sc,chen2024self_para_consistency,jinnai2024mbr_diverse,requeima2024llm_processes}. A single pass already yields a structured frontier prediction, but repeated decoding exposes complementary local modes of the model's output distribution. The four-step procedure above converts this stochastic diversity into better set-level coverage, which is precisely what is measured by hypervolume and IGD\(^{+}\).

Finally, note that MPPF does not change the trained parameters and introduces no external optimizer at test time. It remains an end-to-end neural inference procedure: all candidate points originate from the LLM itself, and post-processing is limited to rounding, deduplication, feasibility checking, non-dominated sorting, and deterministic frontier resampling.

\section{Experiment Settings}
\label{app:exp_settings}

The backbone model used in all experiments is Qwen2.5-7B-Instruct. We fine-tune the model with a maximum context length of 12{,}000 tokens and use bfloat16 training by default. For parameter-efficient fine-tuning, both the LoRA rank and scaling factor are set to 64, with a LoRA dropout of 0.05. LoRA is applied to the attention and MLP projection layers, including \texttt{q\_proj}, \texttt{k\_proj}, \texttt{v\_proj}, \texttt{o\_proj}, \texttt{gate\_proj}, \texttt{up\_proj}, and \texttt{down\_proj}, and the bias option is set to \texttt{lora\_only}. We enable gradient checkpointing during training and apply gradient clipping with \(\texttt{max\_grad\_norm}=1.0\) for stability. Optimization is performed with AdamW (\texttt{adamw\_torch}) using a learning rate of \(2\times 10^{-5}\), weight decay \(0.01\), a linear learning-rate scheduler, and 20 warmup steps. The default per-device batch size is 1, with gradient accumulation steps set to 1, and the model is trained for 1 epoch. For the proposed curriculum objective, we set \(\texttt{phase1\_ratio}=0.15\), \(\texttt{phase2\_ratio}=0.50\), \(\lambda_{\min}^{\mathrm{CE}}=0.4\), \(\beta_{\max}=1.0\), and \(\gamma_{\max}=0.5\). In our implementation, the tokenizer is expanded with the proposed discretized numerical vocabulary, and the added integer-prefix and fractional-suffix tokens are used together with the NGTI initialization strategy. Unless otherwise stated, training uses the chat-style \texttt{messages} format, where loss is computed only on the assistant response span with the response prefix \texttt{<|im\_start|>assistant\textbackslash n}; the codebase also supports an Alpaca-style format as an alternative option. The training set is shuffled with random seed 42, and all sequences are truncated to the maximum sequence length without padding at tokenization time. Checkpoints are saved every 10{,}000 steps, with at most 3 checkpoints retained. The model is trained on 4 NVIDIA A100-SXM4-80GB GPUs on a server equipped with two Intel Xeon Platinum 8362 CPUs at 2.80 GHz.

\section{Baseline Details}
\label{app:baseline_details}

This appendix describes how each baseline method in Section~\ref{sec:baselines} is adapted to our constrained bi-objective continuous setting. Because all baselines must produce a set of $K=20$ feasible decision vectors approximating the Pareto frontier, every method is configured to emit such a set rather than a single solution.

\paragraph{General-purpose and reasoning LLMs.}
All general-purpose and reasoning baselines listed in Table~\ref{tab:api_endpoints} are queried under the same prompting protocol. The prompt and input content are identical to the training data, following Appendix~\ref{app:prompt_template}, with the only difference that the discretized tokens used during training are decoded back into plain numbers, since the models have not learned these special tokens. Predicted vectors are extracted using a deterministic regex parser; malformed outputs count as parse failures and contribute to the empty-feasible-set rate. Reasoning models are allowed their default extended chain-of-thought prior to producing the final frontier output. The exact API model identifier used for each baseline and the base URL for invoking the corresponding model are listed in Table~\ref{tab:api_endpoints}. 

\begin{table}[h]
\centering
\caption{API model identifiers and corresponding base URLs used for the general and reasoning LLM baselines.}
\label{tab:api_endpoints}
\small
\setlength{\tabcolsep}{4pt}
\renewcommand{\arraystretch}{1.05}
\resizebox{\linewidth}{!}{%
\begin{tabular}{lll}
\toprule
Display name & API identifier  & Base URL \\
\midrule
\multicolumn{3}{l}{\textit{General-purpose}} \\
GPT-5.2 Instant      & \texttt{gpt-5.2-chat-latest}        & \texttt{https://api.openai.com/v1} \\
Claude Sonnet 4.5    & \texttt{claude-sonnet-4-5-20250929} & \texttt{https://api.anthropic.com/v1} \\
Claude Haiku 4.5     & \texttt{claude-haiku-4-5-20251001}  & \texttt{https://api.anthropic.com/v1} \\
Gemini 2.5 Flash     & \texttt{gemini-2.5-flash}           & \texttt{https://generativelanguage.googleapis.com/v1beta} \\
DeepSeek-V3 671B     & \texttt{deepseek-chat}              & \texttt{https://api.deepseek.com} \\
Qwen3-235B           & \texttt{qwen3-235b-a22b-instruct-2507}   & \texttt{https://dashscope.aliyuncs.com/compatible-mode/v1} \\
Llama-3.3-70B        & \texttt{llama-3.3-70b-versatile}     & \texttt{https://api.groq.com/openai/v1} \\
Qwen2.5-72B-Instruct & \texttt{Qwen2.5-72B-Instruct}       & \texttt{https://dashscope.aliyuncs.com/compatible-mode/v1} \\
\midrule
\multicolumn{3}{l}{\textit{Reasoning}} \\
GPT-5.2 Thinking          & \texttt{gpt-5.2}                   & \texttt{https://api.openai.com/v1} \\
Claude Haiku 4.5 Thinking & \texttt{claude-haiku-4-5-20251001} & \texttt{https://api.anthropic.com/v1} \\
DeepSeek-R1               & \texttt{deepseek-reasoner}         & \texttt{https://api.deepseek.com} \\
Qwen3-Thinking            & \texttt{qwen3-235b-a22b-thinking-2507}    & \texttt{https://dashscope.aliyuncs.com/compatible-mode/v1} \\
\bottomrule
\end{tabular}%
}
\end{table}

\paragraph{Optimization by PROmpting (OPRO)}
We initialize OPRO~\citep{yang2024opro} with $5$ randomly sampled feasible solutions and ask the LLM to propose better candidates than all previously seen ones. The prompt explicitly displays the previous decision vectors and their objective values, sorted in descending order, and requests the next solution in the same format. Across $4$ outer iterations we accumulate the union of LLM-proposed solutions, then return the non-dominated front of the feasible subset, resampled to $K$ points by arc-length selection in objective space.

\paragraph{PHP.}
For Progressive-Hint Prompting~\citep{zheng2023php}, we run $4$ progressive rounds starting from a random feasible candidate. After each round, the prompt is updated with a hint reflecting whether the previous output was feasible, whether it dominated the current best, and how far it deviates from a single-objective reference. Final reporting follows the same arc-length resampling procedure as OPRO.

\paragraph{LMEA.}
LMEA~\citep{liu2024lmea} treats the LLM as an evolutionary operator. We use a population size of $5$ and $3$ outer generations. The crossover prompt asks the LLM to combine two parent decision vectors by mixing segments, and the mutation prompt asks for a small perturbation that preserves feasibility. Selection is performed by non-dominated sort with crowding-distance tie-breaking. The final population's feasible non-dominated subset is resampled to $K$ points.

\paragraph{SGE.}
SGE~\citep{iklassov2024sge} prompts the LLM to write Python code implementing a heuristic or metaheuristic for the given problem. We instruct the LLM to use only NumPy/SciPy and standard libraries, and execute the generated code in a sandboxed environment with the same per-instance budget as the matched-budget classical comparison. If the generated code runs successfully, we collect its output points; if execution fails, all $K$ slots are recorded as infeasible.

\paragraph{Classical optimizers.}
The classical baselines (Weighted Sum, NBI, NC, NSGA-II, MOEA/D) are implemented on top of pymoo (for evolutionary methods) and Gurobi/MOSEK (for scalarization-based methods, using the same backend solvers as our reference frontier construction). For matched-budget runs, each method is given the same per-instance wall-clock budget as DIPS+MPPF; for full-budget runs, each method runs until its default convergence criterion (Appendix~\ref{app:fullbudget_classical}).

\section{Additional Experimental Details}
\label{app:extra}

This appendix collects the supplementary experiments referenced from the main text. Subsections~\ref{app:solving_time}--\ref{app:reasoning_trace} characterize baseline behavior and inference cost; Subsection~\ref{app:fullbudget_classical} extends the comparison with classical optimizers to unrestricted compute; Subsections~\ref{app:anchor_ablation}--\ref{app:training_stability} provide design ablations and stability analyses; and Subsections~\ref{app:unified_solver}--\ref{app:per_family_results} report fine-grained generalization results.

\subsection{Per-Family Inference Time}
\label{app:solving_time}

This appendix reports per-family inference time, measured as the wall-clock cost from problem input to the final $K=20$ decision vectors, complementing the aggregated ``Avg.\ Time'' column in Table~\ref{tab:main_results}.

\begin{table}[h]
\centering
\caption{Per-family inference time of all baseline methods, in seconds (s) or minutes (m).}
\label{tab:solving_time_app}
\small
\setlength{\tabcolsep}{4pt}
\begin{tabular}{lccccc}
\toprule
Method & SBQP & BOQP & Ridge & Huber & Softplus \\
\midrule
\multicolumn{6}{l}{\textit{General-purpose Language Models}} \\
GPT-5.2 Instant            & 6.4s  & 6.5s  & 5.9s  & 6.7s  & 5.7s \\
Claude Sonnet 4.5          & 7.6s  & 7.4s  & 7.1s  & 7.8s  & 7.0s \\
Claude Haiku 4.5           & 9.3s  & 9.1s  & 8.8s  & 9.4s  & 8.7s \\
Gemini 2.5 Flash           & 5.9s  & 6.1s  & 5.5s  & 6.3s  & 5.4s \\
DeepSeek-V3 671B           & 26.7s & 27.3s & 24.9s & 28.1s & 24.4s \\
Qwen3-235B                 & 15.0s & 15.2s & 13.9s & 15.7s & 13.5s \\
Llama-3.3-70B              & 2.9s  & 2.8s  & 2.6s  & 3.0s  & 2.5s \\
Qwen2.5-72B                & 13.5s & 13.7s & 12.5s & 14.0s & 12.1s \\
\midrule
\multicolumn{6}{l}{\textit{Reasoning Models}} \\
GPT-5.2 Thinking           & 1.7m  & 1.9m  & 1.6m  & 2.0m  & 1.6m \\
Claude Haiku 4.5 Thinking  & 2.3m  & 2.5m  & 2.2m  & 2.6m  & 2.2m \\
DeepSeek-R1                & 5.5m  & 6.0m  & 5.2m  & 6.4m  & 5.2m \\
Qwen3-Thinking             & 4.1m  & 4.5m  & 3.9m  & 4.6m  & 3.7m \\
\midrule
\multicolumn{6}{l}{\textit{Prompt Strategies}} \\
OPRO & 2.2m & 2.4m & 2.1m & 2.5m & 2.1m \\
PHP  & 1.6m & 1.8m & 1.5m & 1.9m & 1.5m \\
LMEA & 5.2m & 5.6m & 5.1m & 5.9m & 5.0m \\
SGE  & 3.7m & 4.0m & 3.5m & 4.1m & 3.5m \\
\bottomrule
\end{tabular}
\end{table}

General-purpose LLMs take seconds per instance, while reasoning models and prompt strategies require minutes due to extended chain-of-thought or repeated LLM calls; LMEA and SGE are the costliest among prompt strategies, as each instance triggers multiple LLM calls.

For DIPS, the cost depends on the number of MPPF passes $N$. Table~\ref{tab:dips_solving_time_scaling} reports timings for $N \in \{1, 2, 4, 8, 16\}$, where $N=1$ is the single-pass DIPS variant and $N \ge 2$ corresponds to DIPS+MPPF.

\begin{table}[h]
\centering
\caption{Per-family inference time of DIPS across MPPF passes $N$, including both autoregressive decoding and candidate fusion.}
\label{tab:dips_solving_time_scaling}
\small
\setlength{\tabcolsep}{6pt}
\begin{tabular}{cccccc}
\toprule
$N$ & SBQP & BOQP & Ridge & Huber & Softplus \\
\midrule
1   & 0.12s & 0.16s & 0.18s & 0.18s & 0.18s \\
2   & 0.22s & 0.27s & 0.30s & 0.30s & 0.30s \\
4   & 0.41s & 0.49s & 0.52s & 0.52s & 0.52s \\
8   & 0.79s & 0.91s & 0.97s & 0.96s & 0.97s \\
16  & 1.63s & 1.81s & 1.85s & 1.85s & 1.83s \\
\bottomrule
\end{tabular}
\end{table}

The cost scales near-linearly in $N$: each pass is an independent autoregressive decoding, and fusion adds only minor overhead (rounding, deduplication, non-dominated sort, and arc-length resampling). Even at $N=16$, DIPS+MPPF completes in under $1.9$ seconds per instance on every family, an order of magnitude faster than the cheapest LLM-based baseline.

\subsection{Example Reasoning Trace}
\label{app:reasoning_trace}

To illustrate why reasoning models struggle on constrained bi-objective continuous Pareto-front generation, we examine the chain-of-thought (CoT) produced by DeepSeek-R1 on a representative SBQP instance with $n=10$ decision variables, where the prompt requests $K=20$ feasible Pareto-optimal vectors and supplies two anchor points. The full CoT exceeds $9{,}000$ tokens, the bulk of which is spent enumerating algorithmic options the model cannot execute and verifying anchor feasibility component by component. We condense it to its key inflection points below, using ellipses (\ldots) for omitted content while preserving the original exploratory and ineffective style.

\begin{quote}\small\itshape
``We are given a bi-objective quadratic problem $\min_x (f_1(x),f_2(x))$ with $f_i(x)=\tfrac{1}{2}x^\top Q_i x + q_i^\top x + c_i$, subject to box bounds and two linear inequalities $Ax\le b$. Two heuristic anchor points near the individual minima are provided\ldots

The natural approach is a weighted scalarization: minimize $wf_1+(1-w)f_2$, which is itself a QP with $Q_w=wQ_1+(1-w)Q_2$ and $q_w=wq_1+(1-w)q_2$. We could sweep $w\in\{0,0.1,\ldots,1\}$\ldots however, solving each such QP requires writing out KKT conditions and inverting $Q_w$, and we are in a text-based environment without an optimization solver. The $\varepsilon$-constraint formulation faces the same obstacle\ldots

Actually, since $Q_1,Q_2$ are diagonal here, $f_1$ and $f_2$ are separable, so the unconstrained minimum of $f_1$ is $x_i^\star=-(q_1)_i/(Q_1)_{ii}$ component-wise. But this may violate the bounds, and once the linear constraints couple the variables we lose separability anyway\ldots

Without a solver we must rely on the anchors. The feasible set is the intersection of a hyperrectangle with two linear half-spaces, hence convex; so any convex combination of two feasible anchors is itself feasible---no clipping required\ldots

But first I should verify that the anchors themselves are feasible. For anchor~1: $x_0=9.1688\in[8.4136,19.7511]$, OK. $x_2=-34.4169$, bounds $[-75.9814,-23.8100]$\ldots $-34.4169<-23.8100$, does that put it below the upper bound? Wait, $-34.4169$ is more negative, so the inequality $-75.9814\le-34.4169\le-23.8100$ does hold\ldots let me re-check\ldots yes, valid. $x_3=-62.6036\in[-99,-34.3284]$\ldots $-62.6036$ is less than $-34.3284$, so within range\ldots [continues element by element through both anchors and both linear constraints]\ldots

Plan: emit the two anchors as solutions~1 and 2, and three convex combinations of the anchors with mixing weights $\{0.25,0.5,0.75\}$ as solutions~3--5. By convexity these are guaranteed feasible\ldots''
\end{quote}

The reasoning never reconciles this five-point plan with the requested $K=20$ output. In the final emission the model pads the five planned vectors with small random perturbations to reach twenty candidates: only $6$ remain feasible after the perturbations push points outside the box, and their objective values cluster along a one-dimensional segment rather than tracing the Pareto front. The bulk of the inference budget is consumed by element-wise bound checks and by enumerating algorithmic options the model cannot actually execute, rather than by producing diverse Pareto-optimal vectors. This illustrates that step-by-step reasoning, though effective on single-answer tasks, scales poorly to set-valued continuous outputs: the model has no internal QP solver, and the convex-combination fallback it converges to is structurally incapable of recovering the Pareto-front geometry.

\subsection{Full-Budget Classical Optimizer Comparison}
\label{app:fullbudget_classical}

This appendix complements the matched-budget study in Section~\ref{sec:classical_reference} by lifting the per-instance time cap, with the aim of measuring the wall-clock cost each method requires to reach its own convergence regime. To this end, each classical baseline is configured with hyperparameters chosen to drive it to a fully converged, high-quality frontier. NSGA-II uses a population size of $300$ and runs for $500$ generations, with simulated binary crossover ($p_c = 0.9$) and polynomial mutation ($p_m = 1/n$). MOEA/D shares the same population size and generation budget, and adopts the Tchebycheff decomposition with a neighborhood size of $20$. For Weighted Sum, NBI, and NC, we formulate $N=320$ scalar subproblems on uniformly spaced weights or reference points and solve each subproblem with Gurobi/MOSEK to its default convergence tolerance. The resulting $320$ candidate solutions are then passed through the same post-processing as MPPF (rounding, deduplication, feasibility filtering, non-dominated sorting, and arc-length resampling) to select the final solution set of $K=20$ Pareto solutions. DIPS+MPPF is reported at $N \in \{1, 2, 4, 8, 16\}$ inference passes, with $N=1$ corresponding to single-pass DIPS.

\begin{table}[h]
\centering
\caption{Full-budget comparison with classical multi-objective optimizers across the five problem families. }
\label{tab:fullbudget_classical}
\scriptsize
\setlength{\tabcolsep}{2pt}
\renewcommand{\arraystretch}{0.95}
\resizebox{\textwidth}{!}{%
\begin{tabular}{lccc ccc ccc ccc ccc c}
\toprule
Method
& \multicolumn{3}{c}{SBQP}
& \multicolumn{3}{c}{BOQP}
& \multicolumn{3}{c}{Ridge}
& \multicolumn{3}{c}{Huber}
& \multicolumn{3}{c}{Softplus}
& \begin{tabular}{c}Avg.\\ Time\end{tabular} \\
\cmidrule(lr){2-4} \cmidrule(lr){5-7} \cmidrule(lr){8-10} \cmidrule(lr){11-13} \cmidrule(lr){14-16}
& fea. & HV & IGD$^{+}$
& fea. & HV & IGD$^{+}$
& fea. & HV & IGD$^{+}$
& fea. & HV & IGD$^{+}$
& fea. & HV & IGD$^{+}$ & \\
\midrule
\multicolumn{17}{l}{\textit{Classical Optimizers (run to convergence)}} \\
Weighted Sum
& 100\% & 0.988 & 0.009
& 100\% & 0.985 & 0.011
& 100\% & 0.990 & 0.007
& 100\% & 0.986 & 0.010
& 100\% & 0.989 & 0.008 & 5.2s \\
NBI
& 100\% & 0.991 & 0.006
& 100\% & 0.988 & 0.008
& 100\% & 0.993 & 0.005
& 100\% & 0.989 & 0.007
& 100\% & 0.992 & 0.006 & 8.8s \\
NC
& 100\% & 0.990 & 0.007
& 100\% & 0.986 & 0.010
& 100\% & 0.992 & 0.006
& 100\% & 0.987 & 0.009
& 100\% & 0.991 & 0.007 & 9.4s \\
NSGA-II
& 100\% & 0.983 & 0.012
& 100\% & 0.978 & 0.016
& 100\% & 0.987 & 0.009
& 100\% & 0.975 & 0.018
& 100\% & 0.981 & 0.013 & 32.4s \\
MOEA/D
& 100\% & 0.986 & 0.010
& 100\% & 0.980 & 0.014
& 100\% & 0.989 & 0.008
& 100\% & 0.977 & 0.017
& 100\% & 0.983 & 0.012 & 25.3s \\
\midrule
\multicolumn{17}{l}{\textit{Ours (DIPS, varying number of MPPF passes $N$)}} \\
DIPS ($N=1$)
& 97\% & 0.982 & 0.017
& 97\% & 0.969 & 0.023
& 98\% & 0.973 & 0.022
& 99\% & 0.968 & 0.025
& 99\% & 0.953 & 0.033 & 0.16s \\
DIPS+MPPF ($N=2$)
& 100\% & 0.994 & 0.008
& 100\% & 0.972 & 0.020
& 100\% & 0.974 & 0.019
& 100\% & 0.970 & 0.022
& 100\% & 0.956 & 0.030 & 0.28s \\
DIPS+MPPF ($N=4$)
& 100\% & 0.998 & 0.004
& 100\% & 0.980 & 0.016
& 100\% & 0.982 & 0.014
& 100\% & 0.972 & 0.019
& 100\% & 0.962 & 0.026 & 0.49s \\
DIPS+MPPF ($N=8$)
& 100\% & 1.000 & 0.004
& 100\% & 0.982 & 0.014
& 100\% & 0.984 & 0.012
& 100\% & 0.975 & 0.017
& 100\% & 0.966 & 0.024 & 0.92s \\
DIPS+MPPF ($N=16$)
& 100\% & 0.999 & 0.004
& 100\% & 0.984 & 0.012
& 100\% & 0.986 & 0.010
& 100\% & 0.977 & 0.015
& 100\% & 0.969 & 0.022 & 1.79s \\
\bottomrule
\end{tabular}%
}
\end{table}

Once each method is allowed to reach its own convergence regime, all baselines—classical optimizers under the hyperparameters above and DIPS+MPPF for sufficiently large $N$—produce frontiers of essentially equivalent quality, since these settings are deliberately chosen to push every method to its optimal regime. The decisive difference lies in cost: even at $N=16$, DIPS+MPPF completes within roughly $1.8$ seconds per instance on every family, whereas the cheapest classical baseline (Weighted Sum) already takes over $5$ seconds, the most accurate ones (NBI, NC) take $8$--$10$ seconds, and the population-based methods (NSGA-II, MOEA/D) are slower still. We emphasize that this is not a strict like-for-like comparison: classical methods consume the exact analytic objective and constraint matrices, whereas DIPS sees only a tokenized text description. The table should therefore be read as evidence that an LLM solver can recover comparable frontier quality much faster, rather than as a direct claim of solver superiority.

\paragraph{Significance analysis.}
To verify that the metric values reported above are not driven by outliers in the test set, we compute $95\%$ bootstrap confidence intervals for feasibility, HV ratio, and IGD$^{+}$ at each setting of $N$, by resampling test instances with replacement ($1{,}000$ resamples). The resulting intervals are reported in Table~\ref{tab:significance_ci}.

\begin{table}[h]
\centering
\caption{$95\%$ bootstrap confidence intervals for feasibility ($\%$), HV ratio, and IGD$^{+}$ across the five problem families at each value of $N \in \{1, 2, 4, 8, 16\}$. Intervals are computed by resampling $1{,}000$ times with replacement over the test set.}
\label{tab:significance_ci}
\scriptsize
\setlength{\tabcolsep}{2pt}
\renewcommand{\arraystretch}{0.95}
\resizebox{\textwidth}{!}{%
\begin{tabular}{l ccc ccc ccc ccc ccc}
\toprule
$N$ & \multicolumn{3}{c}{SBQP} & \multicolumn{3}{c}{BOQP} & \multicolumn{3}{c}{Ridge} & \multicolumn{3}{c}{Huber} & \multicolumn{3}{c}{Softplus} \\
\cmidrule(lr){2-4} \cmidrule(lr){5-7} \cmidrule(lr){8-10} \cmidrule(lr){11-13} \cmidrule(lr){14-16}
& fea. & HV & IGD$^{+}$ & fea. & HV & IGD$^{+}$ & fea. & HV & IGD$^{+}$ & fea. & HV & IGD$^{+}$ & fea. & HV & IGD$^{+}$ \\
\midrule
1  & [95.7, 97.9]\%   & [0.974, 0.989] & [0.013, 0.022] & [95.4, 97.5]\%   & [0.962, 0.975] & [0.020, 0.026] & [96.9, 98.8]\%   & [0.964, 0.979] & [0.019, 0.025] & [98.6, 99.6]\%   & [0.965, 0.971] & [0.023, 0.026] & [98.5, 99.5]\%   & [0.949, 0.957] & [0.031, 0.035] \\
2  & [99.5, 100.0]\%  & [0.988, 0.998] & [0.006, 0.010] & [100.0, 100.0]\% & [0.965, 0.978] & [0.018, 0.023] & [100.0, 100.0]\% & [0.968, 0.980] & [0.016, 0.022] & [100.0, 100.0]\% & [0.966, 0.973] & [0.021, 0.024] & [100.0, 100.0]\% & [0.951, 0.960] & [0.028, 0.033] \\
4  & [99.5, 100.0]\%  & [0.994, 1.000] & [0.004, 0.005] & [100.0, 100.0]\% & [0.975, 0.983] & [0.014, 0.018] & [100.0, 100.0]\% & [0.978, 0.985] & [0.013, 0.016] & [100.0, 100.0]\% & [0.967, 0.975] & [0.018, 0.022] & [100.0, 100.0]\% & [0.959, 0.965] & [0.024, 0.027] \\
8  & [100.0, 100.0]\% & [0.999, 1.000] & [0.004, 0.004] & [100.0, 100.0]\% & [0.978, 0.984] & [0.013, 0.016] & [100.0, 100.0]\% & [0.980, 0.986] & [0.011, 0.014] & [100.0, 100.0]\% & [0.972, 0.977] & [0.016, 0.018] & [100.0, 100.0]\% & [0.963, 0.968] & [0.022, 0.025] \\
16 & [100.0, 100.0]\% & [0.999, 1.000] & [0.003, 0.004] & [100.0, 100.0]\% & [0.983, 0.985] & [0.011, 0.013] & [100.0, 100.0]\% & [0.984, 0.987] & [0.010, 0.011] & [100.0, 100.0]\% & [0.976, 0.979] & [0.015, 0.016] & [100.0, 100.0]\% & [0.966, 0.971] & [0.020, 0.023] \\
\bottomrule
\end{tabular}%
}
\end{table}

The intervals are tight on every metric and every family, and they shrink monotonically as $N$ grows: by $N=8$ the feasibility interval reaches $[100.0, 100.0]\%$ on all five families, and the HV intervals span ranges of width $0.005$--$0.008$. This indicates that the reported gains from MPPF are statistically reliable rather than artifacts of test-instance variability, and that the convergence pattern toward classical-optimizer quality observed in Table~\ref{tab:fullbudget_classical} is a stable property of the model rather than a consequence of fortunate sampling.

\subsection{Anchor-Solution Latent Exploration}
\label{app:anchor_ablation}

The DIPS prompt includes two single-objective anchor solutions, one approximately minimizing $f_1$ and one approximately minimizing $f_2$, computed cheaply at instance generation time. These anchors provide compact cues about the two ends of the trade-off surface and are intended to bias autoregressive generation toward the relevant region of the solution space, analogously to how nearest-neighbor heuristics bias sequence-to-sequence routing solvers~\citep{jiang2025llm_end2end_co}. We attempted to remove the anchors from the input and observed that training collapses, with the model failing to produce structurally valid Pareto-front sequences. The anchors do not constrain the output—the model is still free to generate any $K$ vectors—but they appear to reduce the search burden by supplying implicit endpoints that serve as reliable reference solutions. This echoes the ``heuristic features in the input prompt'' observation in the discrete LLM-CO setting~\citep{jiang2025llm_end2end_co}: lightweight, problem-aware text features can substantially improve learning efficiency without compromising the language-driven interface.

\subsection{Training Stability}
\label{app:training_stability}

To complement the quantitative ablation in Section~\ref{sec:ablation}, we further assess each variant along three qualitative axes: \textbf{Parse Success}, indicating whether the model produces solution sets in the prescribed format; \textbf{Training Stability}, indicating whether optimization proceeds smoothly throughout training; and \textbf{Final Quality}, indicating the converged performance. We label training as Stable when optimization is smooth and outputs remain consistently valid, Unstable when the loss is highly noisy and degenerate generations occur with non-trivial probability, and Collapse when optimization exhibits severe degradation. The results are summarized in Table~\ref{tab:ablation_stability}.

\begin{table}[h]
\centering
\caption{Qualitative training-stability comparison of the ablated variants.}
\label{tab:ablation_stability}
\small
\setlength{\tabcolsep}{5pt}
\renewcommand{\arraystretch}{1.0}
\begin{tabular}{lccc}
\toprule
Method & Parse Success & Training Stability & Final Quality \\
\midrule
Raw Decimal + CE                  & Low    & Unstable            & Low \\
Discretization + Random Init + CE & Medium & Unstable / Collapse & Low \\
Discretization + NGTI + CE        & High   & Stable              & Medium \\
Full DIPS                         & High   & Stable              & High \\
\bottomrule
\end{tabular}
\end{table}

The raw-decimal baseline performs poorly because the model must simultaneously learn the output structure and high-precision continuous values from irregular decimal strings. Discretization substantially simplifies parsing, but random initialization renders the newly introduced numerical tokens difficult to optimize reliably. NGTI mitigates this adaptation issue and restores stable training, yet cross-entropy supervision alone still leaves a clear gap in frontier quality. Only the full DIPS pipeline jointly addresses the three principal bottlenecks: representation difficulty, new-token adaptation, and the numerical insensitivity of token-level supervision.

To further illustrate these training dynamics, Figure~\ref{fig:loss_curves} plots the training loss curves with and without DIPS on three representative families. Across all three, the DIPS variant converges to a markedly lower loss floor, whereas the baseline without DIPS plateaus early at a substantially higher level. This suggests that cross-entropy training alone captures only the coarse geometric structure of the solutions and fails to drive the outputs toward accurate numerical values, consistent with the qualitative assessments in Table~\ref{tab:ablation_stability}.

\begin{figure}[h]
\centering
\begin{minipage}[b]{0.32\textwidth}
    \centering
    \includegraphics[width=\textwidth]{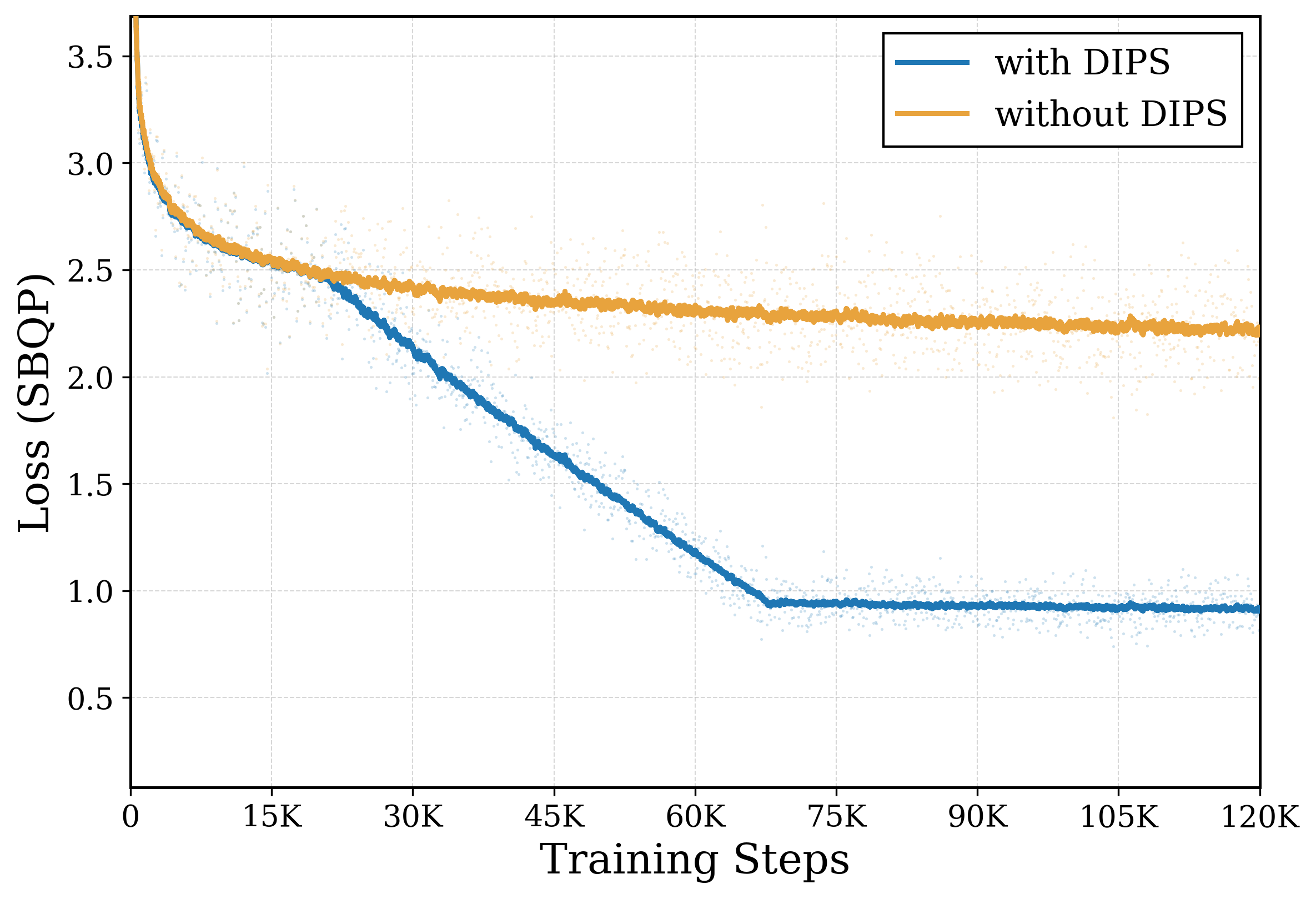}
    \\ \small (a) SBQP
\end{minipage}
\hfill
\begin{minipage}[b]{0.32\textwidth}
    \centering
    \includegraphics[width=\textwidth]{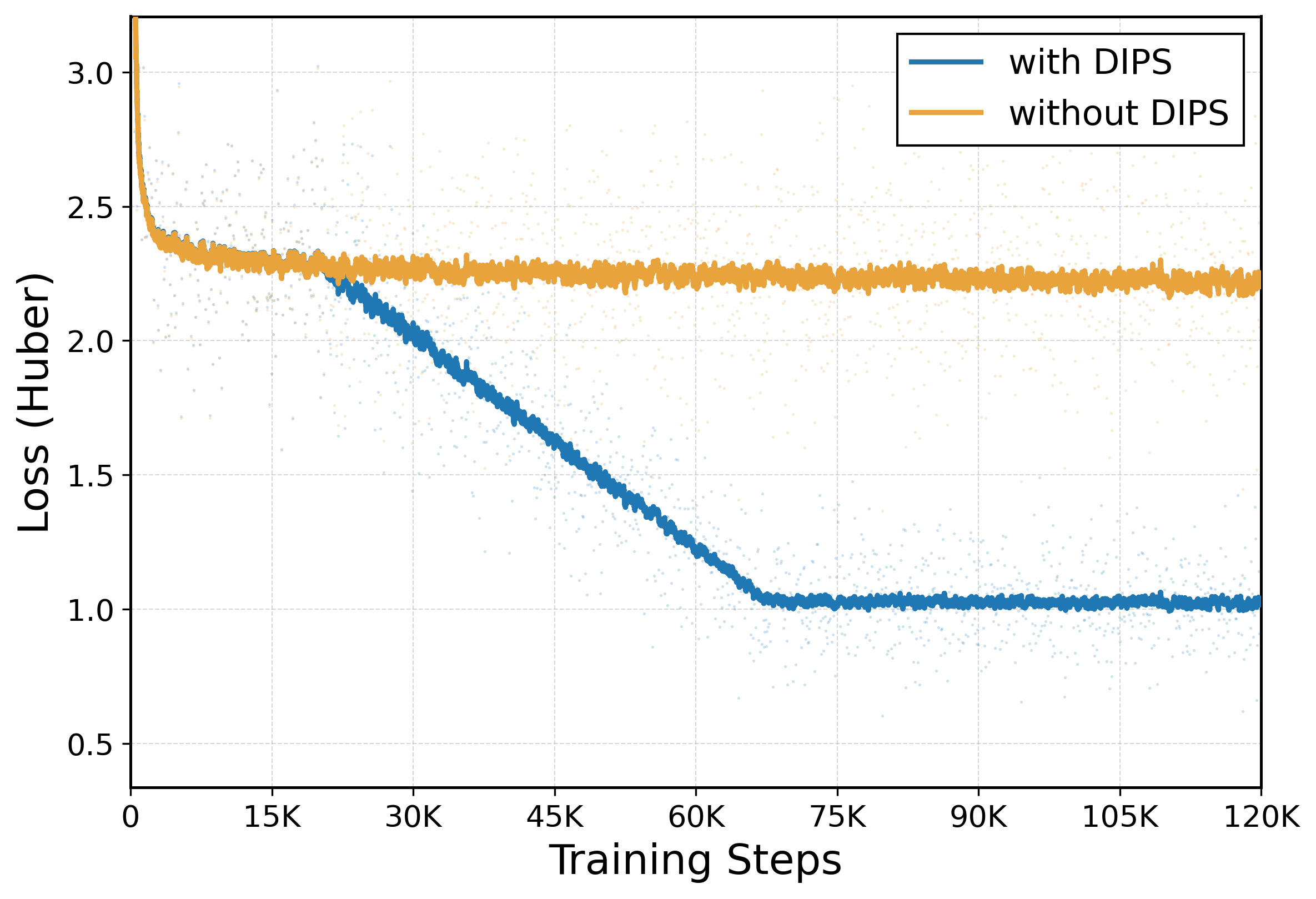}
    \\ \small (b) Huber
\end{minipage}
\hfill
\begin{minipage}[b]{0.32\textwidth}
    \centering
    \includegraphics[width=\textwidth]{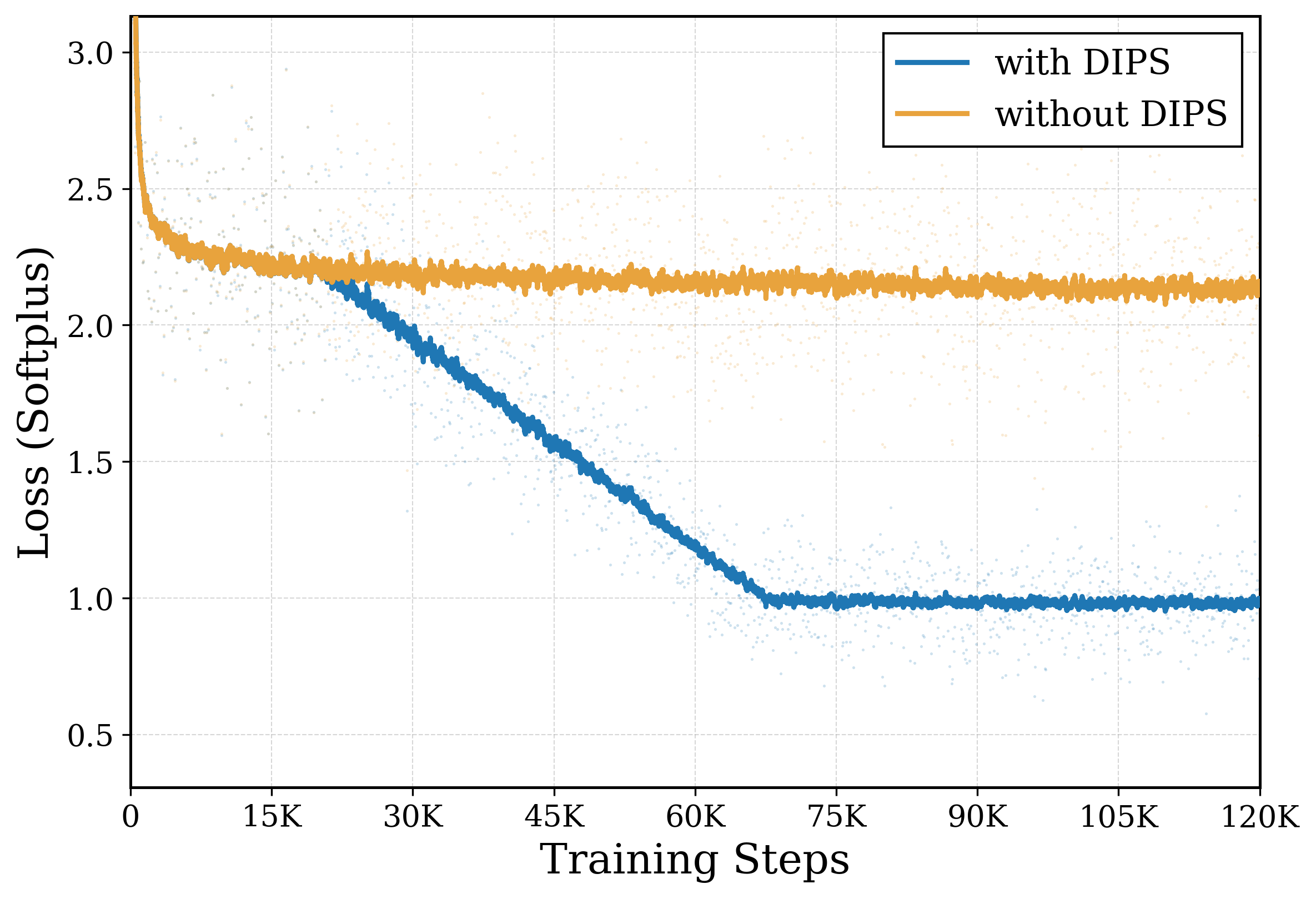}
    \\ \small (c) Softplus
\end{minipage}
\caption{Training loss curves with and without DIPS on three problem families: (a) SBQP; (b) Huber; (c) Softplus.}
\label{fig:loss_curves}
\end{figure}

\subsection{Per-family DIPS vs.\ unified DIPS on the five bi-objective continuous optimization families}
\label{app:unified_solver}

To demonstrate the potential of our approach toward a unified multi-objective optimization solver, we fine-tune a single Qwen2.5-7B model with DIPS on a combined training corpus aggregated from all five bi-objective continuous optimization problem families studied in this work. As shown in Table~\ref{tab:unified_detailed}, the unified model attains strong results in feasibility, HV, and IGD$^{+}$ across every family, performing on par with the per-family experts. These results indicate that training across heterogeneous problem types does not noticeably compromise performance. With further scaling in model size, task diversity, and prompt design, large language models hold considerable promise as foundation models for multi-objective optimization.

\begin{table}[h]
\centering
\caption{Per-family DIPS vs.\ unified DIPS on the five bi-objective continuous optimization families (single-pass inference, $N=1$).}
\label{tab:unified_detailed}
\small
\setlength{\tabcolsep}{4pt}
\begin{tabular}{lccc ccc}
\toprule
Family & \multicolumn{3}{c}{Per-family DIPS} & \multicolumn{3}{c}{Unified DIPS} \\
\cmidrule(lr){2-4} \cmidrule(lr){5-7}
& fea. & HV & IGD$^{+}$ & fea. & HV & IGD$^{+}$ \\
\midrule
SBQP     & 97\% & 0.982 & 0.017 & 100\% & 0.987 & 0.015 \\
BOQP     & 97\% & 0.969 & 0.023 & 100\% & 0.981 & 0.019 \\
Ridge    & 98\% & 0.973 & 0.022 & 100\% & 0.971 & 0.023 \\
Huber    & 99\% & 0.968 & 0.025 & 100\% & 0.971 & 0.024 \\
Softplus & 99\% & 0.953 & 0.033 & 100\% & 0.964 & 0.027 \\
\bottomrule
\end{tabular}
\end{table}

\subsection{Backbone Versatility}
\label{app:backbone_versatility}

This appendix provides the detailed per-family results for the backbone versatility study referenced in Section~\ref{sec:versatility}. The DIPS pipeline is instantiated on Llama-3.1-8B and Gemma-2-9B in addition to the default Qwen2.5-7B-Instruct backbone, with all other components (discretization, NGTI, TPCO, MPPF, hyperparameters) held fixed. We report each backbone under two settings: DIPS, which performs a single autoregressive decoding pass per instance, and DIPS+MPPF with $N=4$ inference passes.

\begin{table}[h]
\centering
\caption{Backbone versatility study. We report per-family fea./HV/IGD$^{+}$ for the three backbones under both DIPS (single pass) and DIPS+MPPF ($N=4$).}
\label{tab:backbone_detailed}
\scriptsize
\setlength{\tabcolsep}{2pt}
\resizebox{\textwidth}{!}{%
\begin{tabular}{lccc ccc ccc ccc ccc}
\toprule
Backbone & \multicolumn{3}{c}{SBQP} & \multicolumn{3}{c}{BOQP} & \multicolumn{3}{c}{Ridge} & \multicolumn{3}{c}{Huber} & \multicolumn{3}{c}{Softplus} \\
\cmidrule(lr){2-4} \cmidrule(lr){5-7} \cmidrule(lr){8-10} \cmidrule(lr){11-13} \cmidrule(lr){14-16}
& fea. & HV & IGD$^{+}$ & fea. & HV & IGD$^{+}$ & fea. & HV & IGD$^{+}$ & fea. & HV & IGD$^{+}$ & fea. & HV & IGD$^{+}$ \\
\midrule
\multicolumn{16}{l}{\textit{DIPS (single pass)}} \\
Qwen2.5-7B   & 97\% & 0.982 & 0.017 & 97\% & 0.969 & 0.023 & 98\% & 0.973 & 0.022 & 99\% & 0.968 & 0.025 & 99\% & 0.953 & 0.033 \\
Llama-3.1-8B & 94\% & 0.951 & 0.039 & 93\% & 0.937 & 0.046 & 95\% & 0.943 & 0.041 & 96\% & 0.934 & 0.048 & 96\% & 0.921 & 0.054 \\
Gemma-2-9B   & 96\% & 0.967 & 0.027 & 95\% & 0.952 & 0.034 & 96\% & 0.958 & 0.031 & 97\% & 0.951 & 0.036 & 97\% & 0.937 & 0.043 \\
\midrule
\multicolumn{16}{l}{\textit{DIPS+MPPF ($N=4$)}} \\
Qwen2.5-7B   & 100\% & 0.998 & 0.004 & 100\% & 0.980 & 0.016 & 100\% & 0.982 & 0.014 & 100\% & 0.972 & 0.019 & 100\% & 0.962 & 0.026 \\
Llama-3.1-8B & 99\%  & 0.972 & 0.018 & 98\%  & 0.954 & 0.029 & 99\%  & 0.957 & 0.027 & 99\%  & 0.948 & 0.033 & 98\%  & 0.937 & 0.040 \\
Gemma-2-9B   & 100\% & 0.987 & 0.011 & 99\%  & 0.967 & 0.022 & 99\%  & 0.971 & 0.020 & 100\% & 0.961 & 0.025 & 99\%  & 0.951 & 0.032 \\
\bottomrule
\end{tabular}%
}
\end{table}

The pattern is consistent across backbones: DIPS+MPPF uniformly outperforms single-pass DIPS, and Qwen2.5-7B retains a small lead over Llama-3.1-8B and Gemma-2-9B that we attribute to differences in numerical pretraining rather than to the DIPS pipeline itself. Across all three backbones, DIPS+MPPF achieves $\geq 98\%$ feasibility on every family, demonstrating that the proposed representation, initialization, and curriculum are not specific to a single model family.

\subsection{Per-Family Fine-Grained Results}
\label{app:per_family_results}

To complement the aggregated comparison in Table~\ref{tab:main_results}, we report a fine-grained analysis of feasibility, HV ratio, and IGD$^{+}$ across the full range of decision-variable dimensions $n \in [10, 20]$, averaged over 100 test instances per dimension. Figure~\ref{fig:per_family_dips} shows the curves of single-pass DIPS, while Figure~\ref{fig:per_family_dips_mppf} shows those of DIPS+MPPF with $N=4$ inference passes.

\begin{figure}[h]
\centering
\begin{minipage}[b]{0.32\textwidth}
    \centering
    \includegraphics[width=\textwidth]{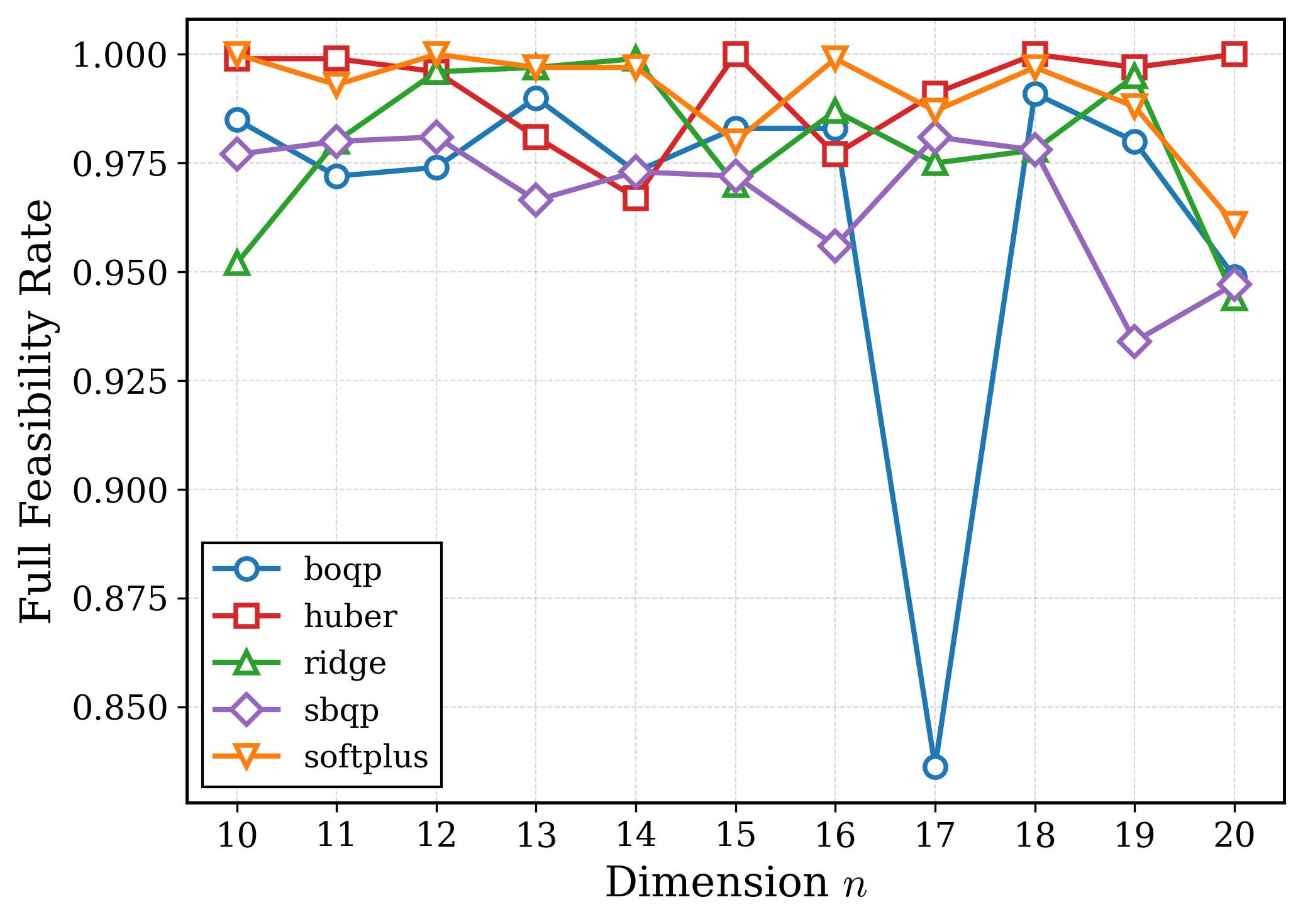}
    \\ \small (a) Full Feasibility Rate
\end{minipage}
\hfill
\begin{minipage}[b]{0.32\textwidth}
    \centering
    \includegraphics[width=\textwidth]{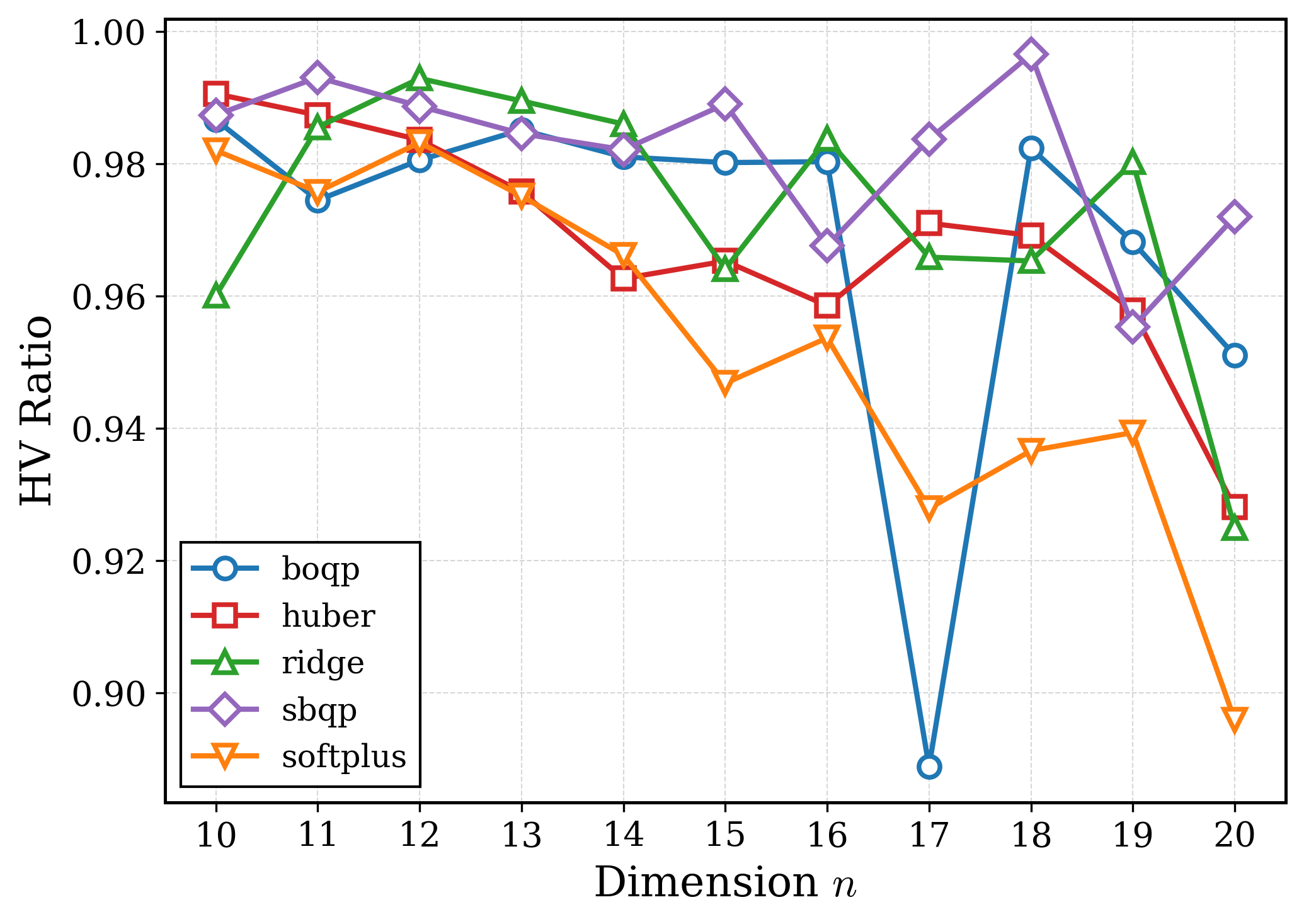}
    \\ \small (b) HV Ratio
\end{minipage}
\hfill
\begin{minipage}[b]{0.32\textwidth}
    \centering
    \includegraphics[width=\textwidth]{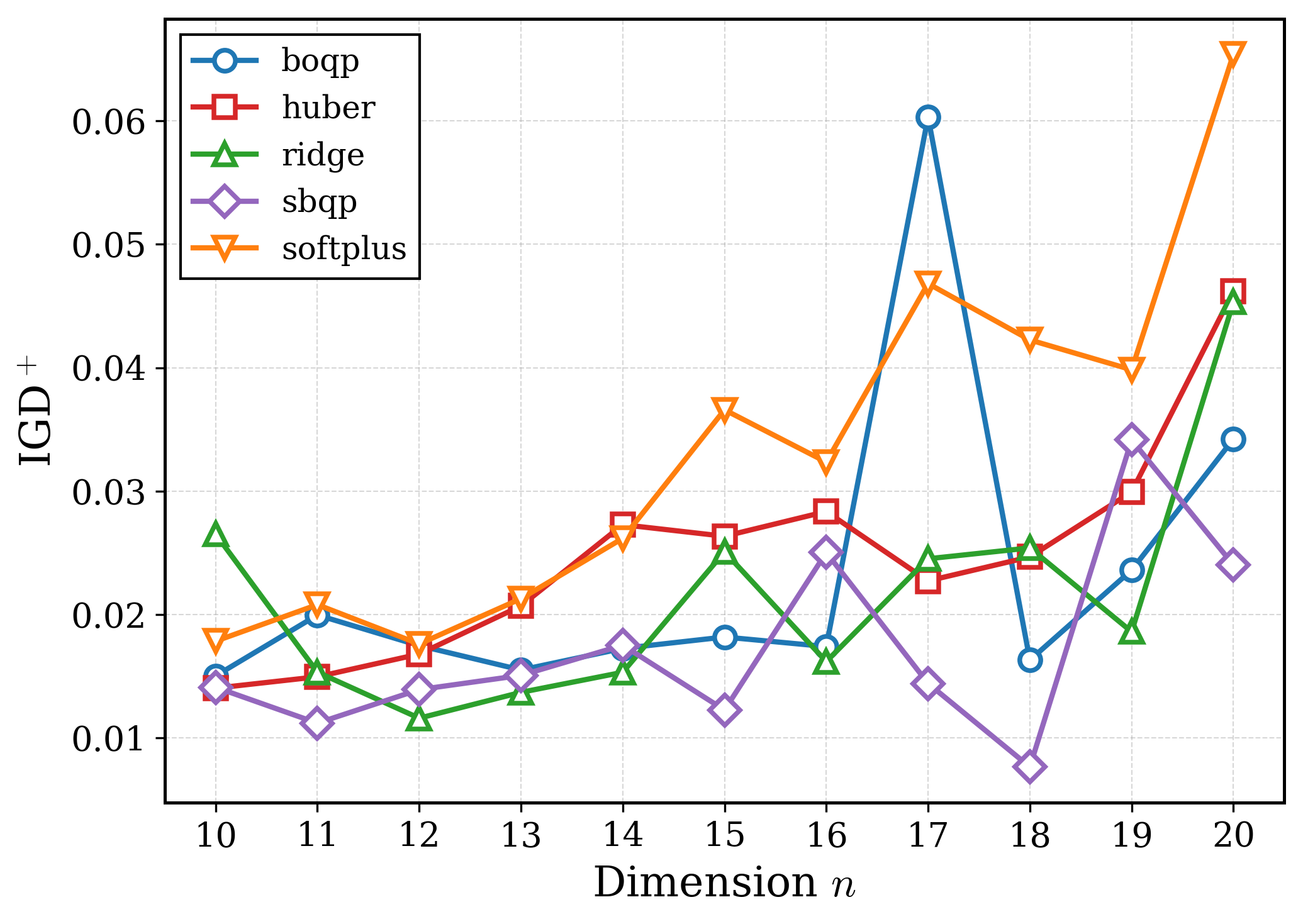}
    \\ \small (c) IGD$^{+}$
\end{minipage}
\caption{Fine-grained per-dimension results of single-pass DIPS across $n \in [10, 20]$: (a) Full Feasibility Rate; (b) HV Ratio; (c) IGD$^{+}$.}
\label{fig:per_family_dips}
\end{figure}

\begin{figure}[h]
\centering
\begin{minipage}[b]{0.32\textwidth}
    \centering
    \includegraphics[width=\textwidth]{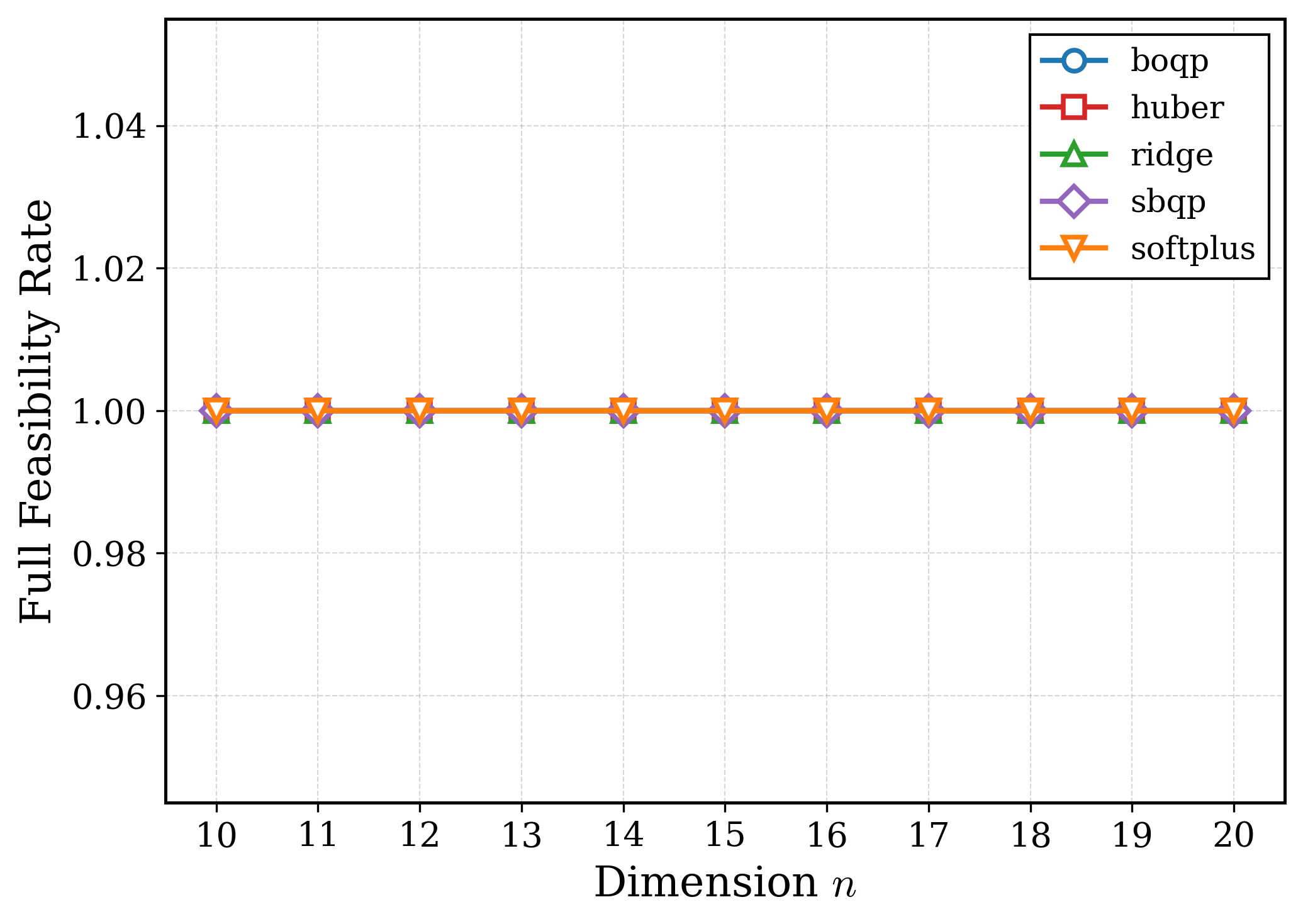}
    \\ \small (a) Full Feasibility Rate
\end{minipage}
\hfill
\begin{minipage}[b]{0.32\textwidth}
    \centering
    \includegraphics[width=\textwidth]{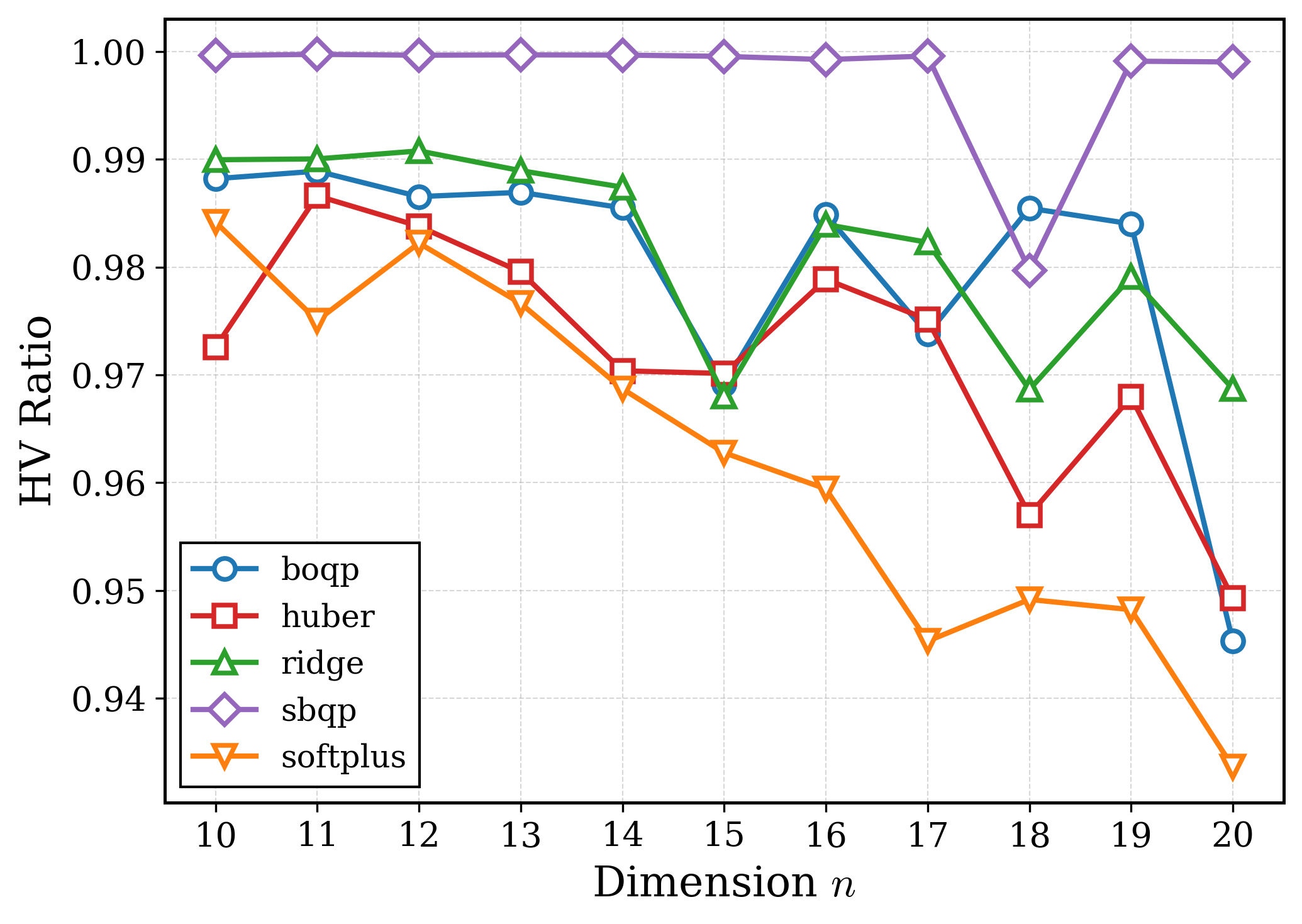}
    \\ \small (b) HV Ratio
\end{minipage}
\hfill
\begin{minipage}[b]{0.32\textwidth}
    \centering
    \includegraphics[width=\textwidth]{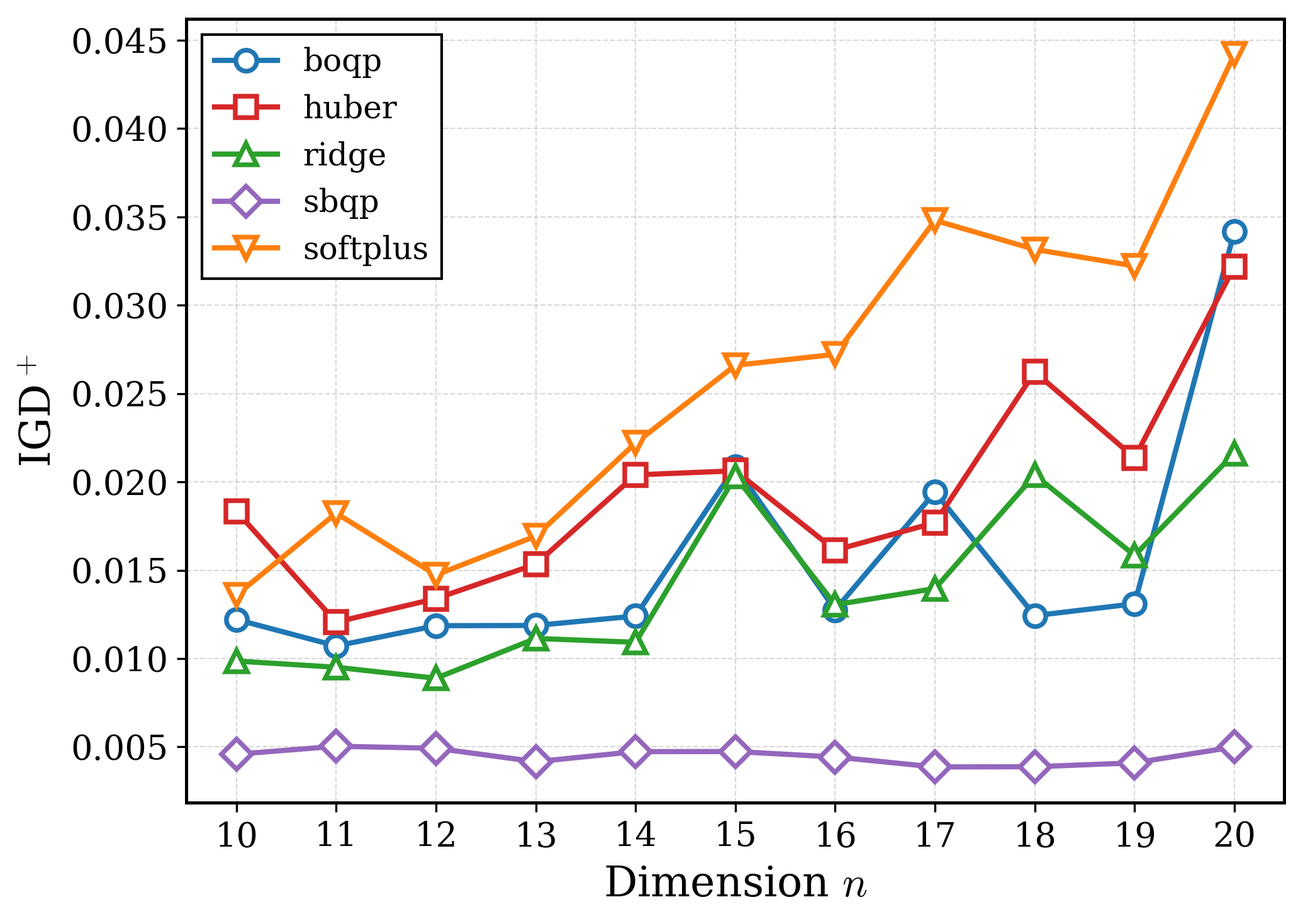}
    \\ \small (c) IGD$^{+}$
\end{minipage}
\caption{Fine-grained per-dimension results of DIPS+MPPF ($N=4$) across $n \in [10, 20]$: (a) Full Feasibility Rate; (b) HV Ratio; (c) IGD$^{+}$.}
\label{fig:per_family_dips_mppf}
\end{figure}

For single-pass DIPS, performance exhibits a mild but consistent degradation as $n$ grows: feasibility and HV ratio gradually decline while IGD$^{+}$ trends upward, reflecting the intrinsic difficulty of higher-dimensional Pareto-front geometry, where the feasible region becomes increasingly anisotropic and the trade-off surface harder to discretize at fixed token budget. This trend is markedly attenuated under DIPS+MPPF in Figure~\ref{fig:per_family_dips_mppf}: feasibility is essentially saturated at $1.0$ across all dimensions, the HV decline flattens, and IGD$^{+}$ remains tightly bounded. By aggregating multiple stochastic decoding passes and reconciling them through non-dominated fusion, MPPF compensates for the higher decoding variance encountered in larger dimensions and recovers a denser, more uniform Pareto front, demonstrating that the framework scales gracefully with problem size.

\subsection{Out-of-Distribution Generalization to Unseen Dimensions}
\label{app:ood_dimensions}

The training distribution covers dimensions $n\in[10,20]$ (Appendix~\ref{app:instance_generation}). To probe behavior strictly outside this range, we evaluate DIPS+MPPF ($N=4$) on freshly generated instances at $n\in\{21,22,23\}$ with $100$ instances per dimension and family. All components are kept identical to the in-distribution setup; no further fine-tuning or prompt adaptation is performed. Results are reported in Table~\ref{tab:ood_dimensions}.

\begin{table}[h]
\centering
\caption{Out-of-distribution evaluation of DIPS+MPPF ($N=4$) on unseen dimensions $n\in\{21,22,23\}$ across all five function families.}
\label{tab:ood_dimensions}
\small
\setlength{\tabcolsep}{6pt}
\renewcommand{\arraystretch}{1.0}
\begin{tabular}{ll ccc}
\toprule
Family & $n$ & fea. & HV & IGD$^{+}$ \\
\midrule
SBQP     & 21 & 100\% & 0.928 & 0.052 \\
         & 22 & 89\%  & 0.847 & 0.115 \\
         & 23 & 71\%  & 0.713 & 0.198 \\
\midrule
BOQP     & 21 & 100\% & 0.901 & 0.067 \\
         & 22 & 82\%  & 0.792 & 0.143 \\
         & 23 & 67\%  & 0.643 & 0.231 \\
\midrule
Ridge    & 21 & 100\% & 0.886 & 0.078 \\
         & 22 & 74\%  & 0.738 & 0.171 \\
         & 23 & 53\%  & 0.572 & 0.281 \\
\midrule
Huber    & 21 & 62\%  & 0.612 & 0.234 \\
         & 22 & 49\%  & 0.443 & 0.347 \\
         & 23 & 25\%  & 0.236 & 0.498 \\
\midrule
Softplus & 21 & 58\%  & 0.567 & 0.258 \\
         & 22 & 21\%  & 0.262 & 0.456 \\
         & 23 & 9\%   & 0.103 & 0.612 \\
\bottomrule
\end{tabular}
\end{table}

As shown in Table~\ref{tab:ood_dimensions}, DIPS+MPPF exhibits partial out-of-distribution generalization to unseen dimensions. On the smoother quadratic and ridge-type families, the model continues to produce feasible Pareto-front approximations at $n=21$ and retains non-trivial performance as the dimension increases, although both feasibility and frontier quality degrade for $n=22,23$. In contrast, the Huber and Softplus families show a much sharper decline, with feasibility and hypervolume dropping substantially beyond the training range. These results suggest that DIPS learns reusable geometric structure within the training regime, but its extrapolation ability remains limited, especially for harder nonlinear objectives. We therefore view unseen-dimension generalization as promising but not yet robust, and identify scaling to higher-dimensional and more nonlinear problem families as an important direction for future work.

\end{document}